\documentclass[12pt]{article}
\usepackage[utf8]{inputenc}
\usepackage{geometry}
\usepackage{hyperref}
\usepackage{setspace}
\usepackage{fix-cm}
\usepackage{authblk}
\usepackage{amsmath}
\usepackage{booktabs}
\usepackage[none]{hyphenat}
\usepackage{tabularx}
\usepackage{float}
\hypersetup{colorlinks=true,allcolors=black}

\usepackage{url}
\usepackage{amsmath, amsfonts, amsthm, amssymb, bm, bbm, verbatim, mathtools, mathrsfs}
\usepackage{color, graphicx, appendix}
\usepackage{etoolbox}
\usepackage{authblk}
\usepackage{array}
\usepackage{romannum}
\usepackage{multirow}

\usepackage{float}
\usepackage{pgf,tikz}
\usetikzlibrary{arrows,shapes.arrows,shapes.geometric,shapes.multipart,fit,automata,decorations.pathmorphing,positioning}
\tikzset{
	>=stealth',
	true/.style={
		rectangle,
		draw=black, very thick,
		text width=6.5em,
		minimum height=2em,
		text centered,
		fill=gray, opacity = 0.5},
	punkt/.style={
		rectangle,
		rounded corners,
		draw=black, very thick,
		text width=6.5em,
		minimum height=2em,
		text centered},
	est/.style={
		circle,
		draw=black, very thick,
		text centered},
	shade/.style={
		circle,
		draw=black, very thick, fill=gray!50,
		text centered},
	weight/.style={
		circle,
		draw=black, very thick,
		text width=6.5em,
		minimum height=2em,
		text centered},
	pil/.style={
		->,
		thick,
		shorten <=2pt,
		shorten >=2pt,},
	double/.style={
		<->,
		thick,
		shorten <=2pt,
		shorten >=2pt,},
	dash/.style={
		dashed,
		thick,
		shorten <=2pt,
		shorten >=2pt,},
	dashdouble/.style={
		<->,
		dashed,
		thick,
		shorten <=2pt,
		shorten >=2pt,}
}

\makeatletter 

\usepackage{epstopdf}
\usepackage[ruled,linesnumbered,algo2e]{algorithm2e}
\usepackage{subcaption}

\usepackage[round]{natbib}

\usepackage{url}            
\usepackage{nicefrac}       
\usepackage{xcolor}
\usepackage{scalerel}
\usepackage{dashrule}
\usepackage{fancyhdr}
\usepackage{tabu}
\usepackage{enumitem}

\def\d{\mathrm{d}}

\def\var{\mathsf{var}}

\def\Cov{\mathsf{Cov}}

\renewcommand{\hat}{\widehat}

\usepackage{xr,xspace}

\usepackage{caption,subcaption,soul}
\usepackage{algpseudocode}
\makeatletter
\usepackage{romanbar}

\theoremstyle{plain}
\newtheorem{theorem}{Theorem}
\newtheorem{lemma}{Lemma}
\newtheorem{proposition}{Proposition}
\newtheorem{corollary}{Corollary}
\theoremstyle{definition}
\newtheorem{definition}{Definition}
\newtheorem{example}{Example}

\newtheorem{problem}{Problem}

\newtheorem{remark}{Remark}

\newtheorem*{remark*}{Remark}
\newtheorem{condition}{Condition}

\usepackage{algorithm}
\usepackage{algpseudocode}

\usepackage{xspace, prettyref}

\newrefformat{eq}{(\ref{#1})}
\newrefformat{chap}{Chapter~\ref{#1}}
\newrefformat{sec}{Section~\ref{#1}}
\newrefformat{alg}{Algorithm~\ref{#1}}
\newrefformat{fig}{Fig.~\ref{#1}}
\newrefformat{tab}{Table~\ref{#1}}
\newrefformat{rmk}{Remark~\ref{#1}}
\newrefformat{clm}{Claim~\ref{#1}}
\newrefformat{def}{Definition~\ref{#1}}
\newrefformat{cor}{Corollary~\ref{#1}}
\newrefformat{lmm}{Lemma~\ref{#1}}
\newrefformat{prop}{Proposition~\ref{#1}}
\newrefformat{prob}{Problem~\ref{#1}}
\newrefformat{app}{Appendix~\ref{#1}}
\newrefformat{hyp}{Hypothesis~\ref{#1}}
\newrefformat{thm}{Theorem~\ref{#1}}

\newcommand{\diag}{\mathsf{diag}}

\newcommand{\Tr}{\mathsf{Tr}}

\renewcommand{\hat}{\widehat}

\newcommand{\bma}{{\bm{a}}}

\newcommand{\bmr}{{\bm{r}}}

\newcommand{\bmu}{{\bm{u}}}
\newcommand{\bmv}{{\bm{v}}}

\newcommand{\bmx}{{\bm{x}}}
\newcommand{\bmy}{{\bm{y}}}

\newcommand{\bmA}{{\bm{A}}}

\newcommand{\bmH}{{\bm{H}}}
\newcommand{\bmI}{{\bm{I}}}

\newcommand{\bmK}{{\bm{K}}}

\newcommand{\bmS}{{\bm{S}}}

\newcommand{\bmX}{{\bm{X}}}

\usepackage{tcolorbox}

\usepackage{scalerel}

\newcommand{\cH}{\mathcal{H}}

\makeatletter
\def\ubar#1{\underline{\sbox\tw@{$#1$}\dp\tw@\z@\box\tw@}}
\makeatother

\numberwithin{example}{section}

\begin{document}
\sloppy

\title{Gradient-Flow Optimization as Dynamic Random-Effects Inference: Testing and Early Stopping with Applications to Deep Learning}

\author[1,*]{Minhao Yao}
\author[2,*]{Ruoyu Wang}
\author[2]{Xihong Lin}
\author[3]{Lin Liu}
\author[4,\textdagger]{Zhonghua Liu}

\affil[1]{Centre for Biomedical Data Science, Duke-NUS Medical School, National University of Singapore}
\affil[2]{Department of Biostatistics, Harvard T.H. Chan School of Public Health, Boston, MA, USA}
\affil[3]{Institute of Natural Sciences, MOE-LSC, School of Mathematical Sciences, CMA-Shanghai, SJTU-Yale Joint Center of Biostatistics and Data Science, Shanghai Jiao Tong University}
\affil[4]{Department of Biostatistics, Columbia University, New York, NY, USA}

\renewcommand\Affilfont{\itshape\small}

\date{}

\maketitle
\pagenumbering{arabic}
\renewcommand{\thepage}{\arabic{page}}

\vspace{-4em}
\begin{center}
\small 
\textsuperscript{*} Equal contribution; 
\textsuperscript{\textdagger}Correspondence: \href{mailto:zl2509@cumc.columbia.edu}{zl2509@cumc.columbia.edu}
\end{center}
\vspace{-1em}

\begin{abstract}
\singlespacing
Gradient-flow optimization is usually viewed as an algorithmic procedure for minimizing empirical loss, with training duration selected by validation or heuristic early stopping rules. We develop a statistical inference framework for gradient-flow training. We show that whenever fitted values evolve through a time-invariant positive semidefinite training operator, the output at each time is equivalent to the best linear unbiased predictor under a corresponding random-effects model. Training time then becomes a variance-component parameter governing variance reallocation from residual noise to structured signal. This turns two training decisions into inferential problems: whether training is needed becomes a variance-component test for signal beyond initialization, and how long to train becomes restricted maximum likelihood (REML) estimation of the training-time variance component. We show that the REML-guided early stopping rule selects the time at which optimized spectral losses become decorrelated from the training-operator eigenvalues. The asymptotic prediction optimality of the REML-guided early stopping time is established for fixed-design in-sample risk and random-design out-of-sample risk. Deep learning models in fixed-kernel gradient regimes provide canonical instantiations for our results. Numerical experiments and a UK Biobank proteomics application show competitive accuracy of the REML-guided early stopping time with reduced reliance on validation splits and repeated checkpoint evaluation.
\end{abstract}


\setstretch{1.7}

\section{Introduction}

Deep learning has achieved remarkable success across a wide range of scientific applications \citep{lecun2015deep}, including protein structure prediction \citep{jumper2021highly,abramson2024accurate}, molecular design \citep{dauparas2022robust}, and mathematical discovery \citep{castelvecchi2026ai}. Despite this success, a principled statistical understanding of deep learning models as predictive procedures remains incomplete \citep{samek2021explaining,bartlett2021deep}. A deep learning predictor is shaped by several interacting components, including the function class, architecture, loss function, optimization algorithm, and stopping rule. Recent statistical overviews have clarified many aspects of approximation, optimization, over-parameterization, and generalization \citep{fan2021selective,bartlett2021deep}, but a statistical interpretation of the training trajectory remains less developed. This limitation is central to the interpretability of deep learning models as statistical procedures and is closely related to the broader concern that deep learning often behaves as black-box models \citep{rudin2019stop}.

A substantial statistical literature studies deep learning models as flexible nonparametric regression estimators, focusing on how network architecture, function-class complexity, and structural assumptions on the target function or covariate distribution determine the accuracy of empirical-risk-minimizing predictors~\citep{schmidt2020nonparametric,bauer2019deep,kohler2021rate,farrell2021deep,jiao2023deep}. These works have greatly advanced the statistical theory of deep learning. However, they provide an incomplete account of the fitted predictor produced by actual training, since they typically focus on the global empirical-risk minimizer, which may not be attainable in practice because of the complicated loss landscape of neural network training~\citep{achour2024loss}. The final predictor is therefore shaped not only by the function class and loss, but also by the optimization trajectory and the time at which training is stopped~\citep{goodfellow2016deep}. Understanding the statistical role of training itself consequently requires a complementary theory of the training path and stopping time. Recently, statistical early stopping methods have emerged for large language models during the reasoning stage \citep{xie2025statistical}; however, to our knowledge, statistically principled stopping rules for training remain largely underexplored.

This leaves both a conceptual and a methodological gap. Conceptually, we lack a statistical characterization of how optimization reallocates variation from residual noise to structured signal along the training path. Methodologically, we lack a corresponding statistical framework that turns the decisions of whether to train and when to stop into inferential problems, rather than relying solely on validation-based or heuristic early-stopping rules~\citep{prechelt1998early,yao2007early,goodfellow2016deep}.
The central question of this paper is whether the training path itself can be given an inferential interpretation. In particular, we ask whether training time can be treated not merely as a tuning parameter, but as an estimable statistical parameter governing signal extraction and model complexity.

We address this question by developing a general optimization--inference duality for fixed-operator squared-error gradient flow. Let $f_t(\bmX)$ denote the vector of fitted values on $n$ training samples $\bmX=(\bmx_1,\cdots,\bmx_n)^\top$ at training time $t$, and let $\bmy$ denote the response vector. Suppose that $f_t(\bmX)$ evolves according to the gradient flow \citep{kovachki2021continuous}
$
\d f_t(\bmX)/\d t=
-\bmH\{f_t(\bmX)-\bmy\}$,
where $\bm H$ is a time-invariant positive semidefinite training operator on the training-sample prediction space. In kernel settings, $\bm H$ is the Gram matrix induced by a positive semidefinite kernel $K(\cdot,\cdot)$, with entries $H_{ij}=K(\bmx_i,\bmx_j)$. The operator $\bmH$ determines the spectral structure of the training dynamics, with its eigenvectors defining residual directions and its eigenvalues controlling their decay rates.  Within our framework, the gradient-flow trajectory is exactly equivalent, at every training time $t$, to the best linear unbiased predictor (BLUP) under a corresponding working random-effects model~\citep{henderson1975best,efron2010large}. Thus, algorithmically, gradient flow is an optimization path; statistically, the same path is a random-effects inference path in which training time acts as a variance-component parameter governing how variation is reallocated from residual noise to structured signal. This operator-level formulation includes early-stopped linear regression \citep{ali2019continuous}, fixed-kernel gradient flow \citep{scholkopf2002learning}, linearized predictors, and, as a canonical deep-learning instance, infinite-width neural networks in the neural tangent kernel (NTK) regime \citep{jacot2018neural}.

The NTK provides a well-established fixed-kernel description of neural-network training in suitable infinite-width regimes \citep{jacot2018neural,du2019gradient,lee2019wide,arora2019exact,huang2020dynamics,montanari2022interpolation,yang2021tensor}. Importantly, this fixed-kernel description is not limited to fully connected networks: NTK training dynamics have been shown to be universal for broad classes of neural network architectures, including convolutional, recurrent, residual, graph, and attention-based architectures \citep{yang2021tensor}. More recently, NTK-based and linearized quantities have also been used to analyze modern AI systems, including transformer fine-tuning and large language model adaptation \citep{malladi2023kernel,afzal2025linearization,li2025efficient}, further illustrating the relevance of fixed-kernel and linearized perspectives.

At the same time, the optimization--inference duality developed in this paper is an operator-level phenomenon, not an NTK-specific one. It holds for any fixed positive semidefinite operator governing squared-error gradient flow. This distinction is also important from the perspective of classical mixed-model theory. Classical mixed-model representations of kernel ridge regression and smoothing splines typically start from a prespecified kernel or an explicit quadratic penalty and lead to a static regularized estimator indexed by a penalty or variance-ratio parameter \citep{wahba1990spline,speed1991comment,ruppert2003semiparametric,liu2007semiparametric}. By contrast, our formulation starts from the unpenalized gradient-flow trajectory. The random-effects covariance, $\exp(t\bm H)-\bm I$, is not imposed by an explicit penalty, but is induced dynamically by the residual-decay operator $\exp(-t\bm H)$ along the training path. In other words, our result is a dynamic random-effects theory for optimization trajectories: the time-indexed BLUP reproduces the early-stopped prediction path, and the training time $t$ becomes an estimable variance-component parameter.

Our dynamic representation turns two basic training decisions into statistical inference problems. First, whether to train becomes a variance-component testing problem: the null hypothesis of no useful training corresponds to the absence of the training-induced random effect, leading to a score test for signal beyond initialization~\citep{lin1997variance}. Second, conditional on training being warranted, the stopping time can be estimated by restricted maximum likelihood (REML)~\citep{henderson1975best,corbeil1976restricted}. In this formulation, early stopping is recast as estimation of a variance component controlling model complexity, rather than as a purely validation-based or heuristic tuning rule.

The REML-guided stopping rule also has an interpretable spectral characterization. In the eigenbasis of the fixed training operator, gradient flow activates directions at different rates: larger-eigenvalue directions are fitted earlier, whereas smaller-eigenvalue directions retain more of their initial residual until later in training. The REML estimating equation selects the training time at which optimized spectral losses become empirically decorrelated from the eigenvalues of the training operator. We further show that, although the REML criterion is derived from a Gaussian working random-effects model, the resulting prediction guarantee does not require Gaussian random effects or errors. Instead, under suitable regularity conditions on the actual data-generating process with independent and identically distributed (i.i.d.) errors without random effects, the REML-guided stopping time achieves oracle-type guarantees for fixed-design in-sample risk and, under additional  regularity conditions, for random-design out-of-sample risk.

We complement the theory with numerical experiments and an application to UK Biobank proteomics data~\citep{sun2023plasma}. Across simulated and real-data examples, the proposed testing procedure detects when training extracts signal beyond initialization, and the REML-guided early stopping rule achieves prediction accuracy comparable to validation-based early stopping while avoiding a separate validation split and reducing repeated checkpoint evaluation. These results demonstrate that the proposed framework provides a statistically interpretable, sample-preserving, and computationally efficient principle for testing and early stopping in settings where (approximately) fixed-operator training dynamics are appropriate.

The remainder of the paper is organized as follows. 
Section~\ref{sec:setup} introduces the data model and the fixed-operator gradient-flow formulation. 
Section~\ref{sec:random-effects} develops the random-effects representation, the variance-component test for training necessity, the REML-guided early stopping rule, the spectral interpretation, and the prediction optimality theory. 
Section~\ref{sec:numerical} presents numerical experiments  in neural-network training settings. 
Section~\ref{sec:UKB} applies the method to UK Biobank proteomics data. Section~\ref{sec:discuss} concludes the paper with a discussion.

\section{Gradient-Flow Optimization with a Fixed Training Operator}
\label{sec:setup}

This section formulates the optimization setting studied in the paper. We first introduce the data model and notation for prediction functions on the training sample. We then describe squared-error gradient flow  through a positive semidefinite training operator. This formulation is intentionally general: it is not tied to a particular model class, and includes linear regression \citep{ali2019continuous}, fixed-kernel methods \citep{scholkopf2002learning,hastie2009elements}, linearized predictors, and neural networks in the neural tangent kernel regime \citep{jacot2018neural,lee2019wide} as special cases. The central optimization object is the fixed-operator gradient-flow trajectory, whose closed-form solution provides the basis for the statistical inference developed in Section~\ref{sec:random-effects}.

\subsection{Data Model and Notation}

Let $\{(\bmx_i,y_i)\}_{i=1}^{n}$ be the training data, generated from
\begin{equation}
\label{eqn:truemodel}
y_i=f_\star(\bmx_i)+\varepsilon_i,
\end{equation}
where $f_\star$ is the ground truth function, $\bmx_i\in\mathbb{R}^d$ is a $d$-dimensional input feature vector, $y_i\in\mathbb{R}$ is the continuous response variable, and $\varepsilon_i$ are i.i.d. error terms with mean zero and variance $\sigma_\varepsilon^2$. Let
$\bmX=(\bmx_1,\ldots,\bmx_n)^\top\in\mathbb{R}^{n\times d}$,
$\bmy=(y_1,\ldots,y_n)^\top$, and
$\bm{\varepsilon}=(\varepsilon_1,\ldots,\varepsilon_n)^\top$
be the design matrix, response vector, and error vector of the training data, respectively. For any scalar function $f$, we write
$f(\bmX)=\{f(\bmx_1),\ldots,f(\bmx_n)\}^{\top}\in \mathbb{R}^n$, i.e., the vector of predictions obtained by applying $f$ row-wise to the design matrix $\bm{X} \in \mathbb{R}^{n \times d}$. Unless otherwise stated, we condition on the design matrix $\bmX$ and treat the input features as fixed; the random-design setting is considered separately in Corollary~\ref{cor:out_sample_optimality}. 
For notational simplicity, we omit hats for intermediate quantities along the training trajectory and reserve the hat notation for final estimators or estimated quantities when distinction is needed.

\subsection{Squared-Error Gradient Flow and Prediction}

We formulate the training dynamics directly through the prediction vector on the training sample. Let $f_t$ denote the predictor at training time $t$. Consider the unnormalized squared-error loss

\begin{equation}
L(f) = \frac{1}{2} \sum_{i=1}^n (y_i - f(\bm{x}_i))^2 = \frac{1}{2} \|\bm{y} - f(\bm{X})\|^2,
\label{eq: squared error loss function}
\end{equation}
where $\|\cdot\|$ is the Euclidean norm on $\mathbb{R}^n$.

The continuous-time formulation below can be viewed as the infinitesimal-learning-rate limit of full-batch gradient descent. Suppose that $f(\bmx)=f_{\bm{\theta}}(\bmx)$ is parametrized by $\bm{\theta}$, where the parameter space may be finite- or infinite-dimensional; see Examples \ref{eg: linear} and \ref{eg: RKHS} below for finite- and infinite-dimensional examples, respectively.  Let $f_m(\bmX) = f_{\bm{\theta}_m}(\bmX)$ be the prediction vector after the $m$-th discrete update with learning rate $\eta$. For full-batch gradient descent applied to the squared-error loss in \eqref{eq: squared error loss function},  the induced prediction update admits the first-order expansion
\begin{equation*}
f_{m+1}(\bm{X})-f_m(\bm{X})
=
\bm{J}_{m}\{-\eta \nabla_{\bm{\theta}} L(f_{\bm{\theta}})\big|_{\bm{\theta}=\bm{\theta}_m}\}
=
-\eta\bm{H}_m(f_m(\bm{X})-\bm{y}),
\end{equation*}
where $\bm{J}_{m} = \nabla_{\bm{\theta}^{\top}} f_{\bm{\theta}}(\bmX)\big|_{\bm{\theta}=\bm{\theta}_m}$ and $\bm{H}_m = \bm{J}_{m}\bm{J}_{m}^{\top}$ is symmetric positive semidefinite. Introducing the continuous training time $t=m\eta$ and taking the limit $\eta\to 0$ yields the gradient-flow description of the training trajectory. This is the standard continuous-time idealization of gradient descent for least-squares problems~\citep{ali2019continuous}; in kernel and neural-network settings, the corresponding prediction-space dynamics are governed by the kernel or tangent-kernel operator~\citep{jacot2018neural,lee2019wide}.
This motivates the following prediction-space gradient-flow ordinary differential equation (ODE):
\begin{equation}
   \frac{\d f_t(\bmX)}{\d t}
   =
   -\bmH(t)\{f_t(\bmX)-\bmy\},
   \label{eq: varying operator flow}
\end{equation}
where $\bmH(t)=\bm{J}(t)\bm{J}(t)^\top\succeq 0$ is a time-dependent training operator that defines the linear map from the  residual to the derivative of the predictions and $\bm{J}(t) = \nabla_{\bm{\theta}^{\top}} f_{\bm{\theta}}(\bmX)\big|_{\bm{\theta}=\bm{\theta}_t}$ is the Jacobian operator. Let $\mathcal T(t)$ denote the residual-evolution operator satisfying
$\d\mathcal T(t)/\d t=-\bmH(t)\mathcal T(t)$ and
$\mathcal T(0)=\bmI$. Then the solution of \eqref{eq: varying operator flow}
with initial condition $f_{t}(\bmX)|_{t=0}=f_0(\bmX)$ is
$
f_t(\bmX)
=
\bm{f}_0(\bm{X}) + \{\bm{I} - \mathcal T(t)\}\{\bmy-f_0(\bmX)\}$.
If $\bmH(s)\bmH(r)=\bmH(r)\bmH(s)$ for all $r,s\in[0,t]$, that is, if the time-dependent training operators are pairwise commuting on $[0,t]$, then
\begin{equation}
 \mathcal T(t)
=
\exp\left\{-\int_0^t \bmH(s)\,\mathrm d s\right\}
=
\exp\{-t\bar{\bmH}(t)\},
\qquad
\bar{\bmH}(t)
=
t^{-1}\int_0^t \bmH(s)\,\mathrm ds.
\label{eq:timevaryingOperator}
\end{equation}
Here $\bar{\bmH}(t)$ is the time-averaged training operator. 

Inspired by continuous-time analyses of early stopping in least-squares regression \citep{ali2019continuous} and fixed-kernel gradient-flow limits in deep learning \citep{lee2019wide}, we focus on the time-invariant case $\bmH(t)\equiv\bmH$. In this case, the residual-evolution operator reduces to $\mathcal T(t)=\exp(-t\bmH)$, yielding
\begin{equation}
   \widehat f_t^{\bmH}(\bmX)
   =
   f_0(\bmX)
   +
   \{\bmI-\exp(-t\bmH)\}
   \{\bmy-f_0(\bmX)\}.
   \label{eq: fixed operator solution}
\end{equation}
When a cross-operator $h(\bmx,\bmX)$ is available for a new input $\bmx$, the corresponding out-of-sample extension of \eqref{eq: fixed operator solution} is
\begin{equation}
\widehat f_t^{\bmH}(\bmx)
=
f_0(\bmx)
+
h(\bmx,\bmX)\bmH^\dagger
\{\bmI-\exp(-t\bmH)\}
\{\bmy-f_0(\bmX)\},
\label{eq: fixed operator estimator}
\end{equation}
where $\bmH^\dagger$ denotes the Moore--Penrose inverse \citep{penrose1955generalized}. If $\bmH$ is positive definite, then $\bmH^\dagger=\bmH^{-1}$. The in-sample expression \eqref{eq: fixed operator solution} is recovered from \eqref{eq: fixed operator estimator} by evaluating $\bmx$ at the training points, in which case $h(\bmx_i,\bmX)$ is the $i$-th row of $\bmH$. The operator $\bm H$ is model-specific. Once a fixed operator $\bm H$ and an initial condition $f_0(\bm X)$ are specified, the gradient-flow ODE uniquely determines the prediction-space trajectory ${\widehat f_t^H(\bm X):t\geq0}$. We now give three concrete examples of such fixed operators.

\begin{example}[Linear regression]\label{eg: linear}
For linear regression, the predictor is $f(\bm{x}) = \bm{x}^\top \bm{\theta}$ with $\bm{\theta} \in \mathbb{R}^d$. The prediction vector on the training set is $\bm{f} = \bm{X}\bm{\theta}$, where $\bm{X} \in \mathbb{R}^{n \times d}$ is the design matrix. The Jacobian operator is $\bm{J}(t) = \bm{X}$, which does not depend on $t$, so the fixed training operator is $\bm{H} = \bm{X}\bm{X}^\top$.
\end{example}

\begin{example}[Unpenalized kernel regression]\label{eg: RKHS}
For kernel regression in a reproducing kernel Hilbert space (RKHS) associated with a kernel \citep{scholkopf2002learning}, consider gradient flow for the unpenalized objective
$
\min_{f \in \mathcal{H}_K} \frac{1}{2} \|\bm{y} - f(\bm{X})\|^2,
$
where $\cH_K$ is the RKHS associated with kernel $K$. Let $\{\mu_j\}_{j=1}^{\infty}$ and $\{\phi_j\}_{j=1}^{\infty}$ denote the eigenvalues and eigenfunctions of the integral operator associated with $K$. Then any function $f\in\cH_K$ can be represented as $f_{\bm{\theta}}(\bmx) = \sum_{j=1}^{\infty}\theta_j\sqrt{\mu_j}\phi_j(\bmx)$ with $\bm{\theta} = (\theta_j)_{j=1}^{\infty}$ satisfying $\sum_{j=1}^{\infty}\theta_j^2 < \infty$.
Under this parametrization, the Jacobian operator is $\bm{J}(t) = (\sqrt{\mu}_j\phi_j(\bmX))_{j=1}^{\infty}$, which is independent of $t$. Consequently, the gradient-flow dynamics for the above objective have the fixed training operator $\bmH=\bmK$, where $\bmK=(K(\bmx_i,\bmx_j))_{i,j=1}^{n}$ is the kernel Gram matrix. With zero initialization $f_0(\bmX)=\bm0$, the resulting early-stopped predictor is
$f_t(\bmX) = (\bm{I} - e^{-t\bm{K}})\bm{y}$. In this example, no explicit penalty term is added; regularization arises solely from early stopping at time $t$. This differs from kernel ridge regression, which adds the explicit penalty term $\rho\|f\|_{\mathcal{H}_K}^2$ to the objective function~\citep{liu2007semiparametric}.
\end{example}

\begin{example}[Deep learning architectures in the NTK regime]
For a neural network architecture in the infinite-width NTK regime, the neural tangent kernel (NTK) converges to a deterministic kernel $h_\infty(\bm{x}, \bm{x}')$ \citep{jacot2018neural,lee2019wide}. This fixed-kernel description is not limited to fully connected networks; NTK training dynamics have been shown to be architecturally universal for broad classes of neural network architectures expressible through tensor programs~\citep{yang2020tensor,yang2021tensor}, including modern architectures such as convolutional, recurrent, residual \citep{he2016deep}, graph, and attention-based networks \citep{vaswani2017attention}. The NTK Gram matrix on the training set is $\bm{H}_\infty = (h_{\infty}(\bm{x}_i, \bm{x}_j))_{i,j=1}^n$. In this regime, the NTK remains constant during training, which implies that the training operator is $\bm{H} = \bm{H}_\infty$ in the gradient-flow equation.
\end{example}

\section{Random-Effects Representation of Fixed-Operator Training}
\label{sec:random-effects}
This section develops the random-effects inference framework for fixed-operator gradient flow. We first establish the exact equivalence between the gradient-flow trajectory and the BLUP under a working random-effects model. We then use this equivalence to construct a variance-component test for signal beyond initialization and a REML-guided estimator of the stopping time. Finally, we derive the spectral interpretation of the REML stopping rule, define the resulting effective degrees of freedom, and establish oracle-type prediction guarantees.

\subsection{Fixed-Operator Gradient Flow as BLUP}
We now establish an exact random-effects representation for the fixed-operator gradient-flow trajectory in equation \eqref{eq: fixed operator solution}. Let $\Tr(\bmA)$ denote the trace of a matrix $\bmA$, let $\bm{I}$ denote the $n\times n$ identity matrix, and let $\gamma_t = n^{-1}\Tr\{\exp(t\bmH)\}$ be the average eigenvalue of $\exp(t\bmH)$. The following theorem shows that this trajectory is equivalent to the BLUP under a time-indexed Gaussian working random-effects model. The proof is given in the Supplementary Materials.

\begin{theorem}[Equivalence between fixed-operator gradient flow and BLUP]
\label{thm: equivalence}
For each training time $t\geq 0$, the fitted-value trajectory $\widehat f_t^{\bmH}(\bmX)$ in \eqref{eq: fixed operator solution} satisfies
\begin{equation*}
    \widehat f_t^{\bmH}(\bmX)
    =
    f_0(\bmX)
    +
    \widehat{\bm{u}}_t,
\end{equation*}
where $\widehat{\bm{u}}_t$ is the BLUP  of the random effects $\bm{u}_t$ in the following working random-effects model:
\begin{equation}
    \begin{split}
        \bmy &= f_0(\bmX) + \bm{u}_t + \bm{\varepsilon}_{t}, \\
        \bm{u}_t &= (u_{1,t},\cdots,u_{n,t})^\top
        \sim
        N\left(\bm{0},
        \sigma_{\varepsilon,t}^2\{\exp(t\bmH)-\bm{I}\}\right), \\
        \bm{\varepsilon}_{t}
        &=
        (\varepsilon_{t,1},\cdots,\varepsilon_{t,n})^\top
        \sim
        N(\bm{0},\sigma_{\varepsilon,t}^2\bm{I}),
    \end{split}
    \label{eq: model RE}
\end{equation}
where $\bm{u}_{t}$ and $\bm{\varepsilon}_{t}$ are independent and
$\sigma_{\varepsilon,t}^{2}
=
\gamma_t^{-1}\sigma_{\varepsilon}^{2}$
for some constant $\sigma_{\varepsilon}^{2}>0$.
\end{theorem}

The Gaussian specification in \eqref{eq: model RE} is stronger than needed for the BLUP equivalence in Theorem \ref{thm: equivalence}. The BLUP depends only on the specified first two moments, namely the mean and covariance structure of $(\bm{u}_t^{\top}, \bm{\varepsilon}_t^{\top})^{\top}$. Thus, the equivalence between the fixed-operator gradient-flow trajectory and the BLUP is an algebraic covariance identity. The Gaussian working model will be used later as a likelihood device for the variance-component score test in Section \ref{sec:test} and REML estimation in Section \ref{subsec: REML-guided time}.

\begin{remark}[Comparison with explicit quadratic-penalty estimators]
Classical estimators with explicit quadratic penalties, including penalized splines \citep{wahba1990spline, speed1991comment,ruppert2003semiparametric,scholkopf2002learning,hastie2009elements} and kernel ridge regression \citep{liu2007semiparametric}, also admit mixed-model representations. However, these representations are static: they arise from an explicitly penalized objective, 
typically of the form
$
\min_f \sum_{i=1}^n\{y_i-f(\bmx_i)\}^2+\rho\mathcal P(f),
$
and describe the final regularized estimator through a covariance of the form $\bm C_\rho=\rho^{-1}\bm C$, up to a variance scale, where $\bm C$ is determined by the penalty $\mathcal P$.

Our construction differs in two key respects. First, it starts from the unpenalized squared-error objective and regularizes implicitly through early stopping. Second, it gives a time-indexed random-effects representation of the entire gradient-flow trajectory: for every $t$, the BLUP reproduces the corresponding early-stopped predictor. The covariance $\bm C_t=\exp(t\bm H)-\bm I$, equivalently the residual-decay operator $\exp(-t\bm H)$, is induced by the training dynamics rather than by a static penalty. This dynamic covariance is central to the convex REML profile objective in Proposition~\ref{prop: existence uniqueness empirical REML} and to the prediction guarantees in Theorem~\ref{thm: test error} and Corollary~\ref{cor:out_sample_optimality}. These distinctions are summarized in Table~\ref{tab:compare} in the Supplementary Materials.
\end{remark}

\subsection{Testing Whether Training Is Needed}
\label{sec:test}

A key question in gradient-flow optimization is whether training from the initialized predictor extracts predictive structure in $\bmy$ from $\bmX$. This question can be formulated through the random-effects representation in \eqref{eq: model RE}. Under the working model, the absence of trainable structure beyond initialization corresponds to a zero training-induced random effect, $\bmu_t=0$, or equivalently $\bmy-f_0(\bmX)=\bm{\varepsilon}_t$. Under the data-generating model \eqref{eqn:truemodel}, there is no need to train under the no-signal null which is naturally expressed, after centering, as $f_\star(\bmX)=0$.

The above two null formulations need not coincide when the initialized prediction $f_0(\bmX)$ is nonzero. To unify these null formulations, we therefore construct the score test after removing the contribution of the initialization, so that the null behavior is not driven by $f_0(\bmX)$. Specifically, assuming $f_0(\bmX)\neq \bm{0}$, we project onto the subspace orthogonal to $f_{0}(\bmX)$ using
$\bm{\Pi} = \bm{I} - (f_0(\bm{X})^{\top}f_0(\bm{X}))^{-1}f_0(\bm{X})f_0(\bm{X})^{\top}$,
and construct the score test from the resulting residual likelihood. Since $\bm{\Pi}$ projects onto an $n-1$ dimensional subspace, it has $n-1$ non-zero eigenvalues. Let $\bm{M}\in \mathbb{R}^{n\times (n-1)}$ be the matrix whose columns are the eigenvectors corresponding to the non-zero eigenvalues of $\bm{\Pi}$. Then $\bm{\Pi}=\bm{M}\bm{M}^\top$ and $\bm{I}_{n-1} = \bm{M}^\top\bm{M}$. Applying this orthogonal transformation to the random-effects model \eqref{eq: model RE} gives
$\widetilde{\bm{y}} = \bm{M}^{\top}\bm{u}_t + \bm{M}^{\top}\bm{\varepsilon}_t$,
with $\Cov(\widetilde{\bm{y}}) = \sigma_{\varepsilon,t}^{2}\bm{M}^{\top}\exp(t\bmH)\bm{M}$.
In this transformed model, the initialized prediction has been removed, so inference targets predictive structure revealed by training beyond initialization.

The residual log-likelihood function after removing the initialized prediction is
\begin{equation}
\begin{aligned}
    \widetilde{\ell}_R(t,\sigma_{\varepsilon,t}^2) &=  - \frac{1}{2} \log  \det\left[\sigma_{\varepsilon,t}^2 \bm{M}^\top \exp(t\bm{H}) \bm{M}\right] - \frac{1}{2 \sigma_{\varepsilon,t}^2} \widetilde{\bmy}^\top \left(\bm{M}^\top\exp(t\bm{H}) \bm{M} \right)^{-1} \widetilde{\bmy}.
\end{aligned} \label{eq: REML projected}
\end{equation}
Then, testing whether training is statistically necessary can be formulated as the hypothesis testing problem $H_0:t=0$ versus $H_1:t>0$. Under the random-effects representation in \eqref{eq: model RE}, the null hypothesis $t=0$ corresponds to the absence of predictive structure extracted beyond initialization. Following~\citet{lin1997variance}, we obtain the following score test statistic 
\begin{equation*}
    \widetilde{T} = \frac{(n-1)\widetilde{\bmy}^\top  \widetilde{\bm{H}} \widetilde{\bmy}}{\widetilde{\bmy}^\top\widetilde{\bmy}},
\end{equation*}
where $\widetilde{\bm{H}} = \bm{M}^\top\bm{H}\bm{M}$. Under the Gaussian working null model, the $p$-value can be computed using the exact method \citep{davies1980algorithm}.

\subsection{REML-Guided Estimation of the Stopping Time}\label{subsec: REML-guided time}

We next use the random-effects representation in \eqref{eq: model RE} to derive a REML-guided early-stopping rule. The working model contains two unknown quantities, the training time $t$ and the noise variance $\sigma_{\varepsilon,t}^2$. We estimate them by restricted maximum likelihood, with $t$ serving as the variance-component parameter that indexes the training trajectory. This leads to the following restricted log-likelihood:
\begin{equation*}
\begin{aligned}
    \ell_R(t,\sigma_{\varepsilon,t}^2) 
    &= -\frac{n}{2} \log (\sigma_{\varepsilon,t}^2) - \frac{1}{2} \log  \det\left[\exp(t \bm{H})\right] - \frac{1}{2 \sigma_{\varepsilon,t}^2} (\bmy - f_{0}(\bmX))^\top \exp(-t \bm{H}) (\bmy - f_{0}(\bmX)).
\end{aligned} \label{eq: REML}
\end{equation*}

For a fixed $t$, maximizing $\ell_R(t,\sigma_{\varepsilon,t}^2)$ with respect to $\sigma_{\varepsilon,t}^2$ gives 
$\widehat{\sigma}_{\varepsilon,t}^2 = n^{-1} (\bmy - f_{0}(\bmX))^\top \exp(-t \bm{H}) (\bmy - f_{0}(\bmX))$.
Substituting this expression back into the restricted log-likelihood shows that maximizing $ \ell_R(t,\widehat{\sigma}_{\varepsilon,t}^2)$ is equivalent to minimizing
\begin{equation*}
    Q(t) = n \log \left( (\bmy - f_{0}(\bmX))^\top \exp(-t \bm{H}) (\bmy - f_{0}(\bmX)) \right) + \log \det \left[ \exp(t \bm{H}) \right].
\end{equation*}
Let $\{\lambda_k,\bm{v}_k\}_{k=1}^n$ be the eigenvalues and eigenvectors of the fixed training operator $\bm{H}$, sorted in descending order. Define the projection coefficient of residual $\bmy - f_{0}(\bmX)$ onto the $k$-th eigenspace as $c_k=\bm{v}_k^\top (\bmy - f_{0}(\bmX))$. Note that $\bmy - f_{0}(\bmX)=\sum_{k=1}^n c_k\bm{v}_k$. Then, we have
\begin{equation}\label{eq: REML loss}
    Q(t) = n \log \left(\sum_{k=1}^n c_k^2 \exp(-t\lambda_k) \right) + t \sum_{k=1}^n \lambda_k, \quad t \geq 0. 
\end{equation}
Differentiating $Q(t)$ with respect to $t$ and setting $Q'(t)=0$ gives:
\begin{equation*}
    Q'(t) =  - n \frac{\sum_{k=1}^n \lambda_k c_k^2 \exp(-t\lambda_k) }{\sum_{k=1}^n  c_k^2 \exp(-t\lambda_k)} + \sum_{k=1}^n \lambda_k = 0.
\end{equation*}
Rearranging this first-order condition, the REML estimate of the stopping time is obtained as the solution to the following estimating equation
\begin{equation}\label{eq: REML EE}
    \Psi_n(t)=\frac{1}{n}\sum_{k=1}^{n}\Big(\lambda_k - \frac{1}{n}\sum_{j=1}^n \lambda_j\Big) \Big\{c_k^2 \exp(-t\lambda_k) - \frac{1}{n}\sum_{j=1}^n c_j^2 \exp(-t\lambda_j)\Big\} = 0.
\end{equation}

The estimating equation \eqref{eq: REML EE} can also be understood without relying directly on the Gaussian likelihood. Let
$\bm r_0=\bmy-f_0(\bmX)$
denote the initial residual vector for the training sample of size $n$, and let $\bm H
=
\bm V\bm\Lambda\bm V^\top$, $\bm\Lambda=\diag(\lambda_1,\dots,\lambda_n)$,
$\bm V=(\bm v_1,\dots,\bm v_n)$,
be the eigendecomposition of $\bm H$, where
$\lambda_1\geq \lambda_2\geq \cdots \geq \lambda_n\geq 0$. Then
$\exp(-t\bm H)
=
\bm V\exp(-t\bm\Lambda)\bm V^\top$,
$
\exp(-t\bm\Lambda)
=
\diag\{\exp(-t\lambda_1),\dots,\exp(-t\lambda_n)\}$. For a candidate training time $t$, define the vector of eigen-residuals $\mathrm{ER}(t)$ as follows 
\begin{equation*}
\mathrm{ER}(t)
=
\exp(-t\bm\Lambda/2)\bm V^\top\bm r_0,
\end{equation*}
where the $k$-th eigen-residual is
$\mathrm{ER}_k(t)
=
\exp(-t\lambda_k/2)\bm v_k^\top\bm r_0, k=1,\dots,n$.

The REML early stopping time is the point at which the squared eigen-residuals become empirically uncorrelated with their corresponding eigenvalues. Equivalently, it solves the moment condition
\begin{equation*}
    \Psi_n(t)
    =
    \frac{1}{n}\sum_{k=1}^n
    \Big(\lambda_k-\frac{1}{n}\sum_{j=1}^n\lambda_j\Big)
    \Big\{
    \mathrm{ER}_k^2(t)
    -
    \frac{1}{n}\sum_{j=1}^n\mathrm{ER}_j^2(t)
    \Big\}
    =
    0.
\end{equation*}
Thus, the REML stopping time is the balance point at which this Empirical Spectral Covariance (ESC) $\Psi_n(t)$ is zero. This moment-based interpretation does not rely directly on the Gaussian likelihood and motivates the spectral loss-decorrelation view developed in Section \ref{subsec:spectral}.

Next, we establish sufficient conditions for the existence and uniqueness of the empirical REML-guided early stopping time $\widehat{t}^{\mathrm{REML}}_n$, defined as a solution to the empirical estimating equation \eqref{eq: REML EE}.
By construction, $\widehat{t}^{\mathrm{REML}}_n$ is any solution to $Q'(t)=0$.
To characterize the shape of $Q(t)$, define $p_k(t)=c_k^2\exp(-t\lambda_k)/\{\sum_{j=1}^n c_j^2\exp(-t\lambda_j)\}$.
Then, $p_1(t),\ldots,p_n(t)$  form a probability distribution $\mathcal{P}^\lambda_t$ on $\{1,\ldots,n\}$, and straightforward calculation yields
\begin{equation*}
Q''(t) = n\sum_{k=1}^n p_k(t)\Big\{\lambda_k-\sum_{j=1}^n p_j(t)\lambda_j\Big\}^2 = n\,\var_{\mathcal{P}^\lambda_t}(\lambda) \geq 0.
\label{eq: supp Q2}
\end{equation*}
Because $Q''(t)$ is non-negative, $Q(t)$ is convex on $[0,\infty)$. The convexity of $Q(t)$ is illustrated empirically in Supplementary Figure~\ref{fig:supp_Q_convexity}(b). Subsequently, let $\widehat{t}^{\rm REML}_n$ be a minimizer of $Q(t)$ over $[0, \infty)$. The following proposition gives  sufficient conditions for the existence and uniqueness of $\widehat{t}^{\rm REML}_n$.

\begin{proposition}
\label{prop: existence uniqueness empirical REML}
Suppose that (i) $Q'(0)<0$; (ii) there exist $k,j\in \{1,\cdots,n\}$ such that $\lambda_k\neq \lambda_j$. Assume $\varepsilon_i$ is a continuous random variable for $i = 1,\dots,n$.
Then, with probability one, there exists a unique $\widehat{t}^{\rm REML}_n\in(0,\infty)$ such that
$Q'(\widehat{t}^{\rm REML}_n)=0$.
Moreover, $\widehat{t}^{\rm REML}_n$ is the unique minimizer of $Q(t)$ over $t\in(0,\infty)$. 
\end{proposition}

\begin{remark}
The uniqueness result exploits the exponential covariance structure of the random effect term
$\bm{u}_t \sim N(\bm{0},\sigma_{\varepsilon,t}^2 ( \exp(t\bm{H}) - \bm{I}))$.
This structure yields the convexity of the REML profile objective and hence the uniqueness of the REML stopping time. In contrast, REML criteria for smoothing-parameter selection in penalized least squares may exhibit multiple local optima \citep{reiss2009smoothing}, as also illustrated by a numerical example in Supplementary Figure \ref{fig:supp_Q_convexity}(a).
\end{remark}

\subsection{Spectral Loss Decorrelation in Training Dynamics}
\label{subsec:spectral}
In this section, we develop a geometric and spectral interpretation of the fixed-operator gradient-flow training dynamics. 
We show that, under the random-effects representation, training progressively reallocates variance from residual noise to a fixed-operator-structured random effect and acts as a coordinate-wise spectral regularization procedure in the fixed-operator eigenbasis.  This interpretation leads to a new  notion of spectral loss decorrelation: the REML estimating equation selects the stopping time at which the optimized spectral losses are empirically decorrelated from the eigenvalues of $\bmH$.  The spectral interpretation also yields a natural effective degrees-of-freedom (edf) measure for the trained model, with each eigendirection of $\bmH$ contributing according to its training-time-dependent activation. 
Together, these results clarify how model complexity evolves during training and provide a statistical rationale for the REML-guided early-stopping criterion.

We first examine how training time reallocates variance between residual noise and the training-induced random effect in the working model. Model \eqref{eq: model RE} uses the random effect $\bm{u}_{t}$ to model the difference $f_\star(\bm{X}) - f_{0}(\bm{X})$.  The variance scale then satisfies $\sigma_{\varepsilon, t}^{2} = n\sigma_{\varepsilon}^{2} / \sum_{k = 1}^{n}\exp(t\lambda_k)$. Under this formulation, the total marginal variance $\sigma^2_{\text{Total}}=\Tr\{\var(\bm{\varepsilon}_{t})\} + \Tr\{\var(\bm{u}_{t})\} = n\sigma_{\varepsilon}^{2}$ remains constant for all $t \geq 0$. In the random-effects model, the covariance of the random effect admits the following spectral decomposition
\begin{equation*}
    \var(\bmu_t) = \sigma_{\varepsilon,t}^2 ( \exp(t\bm{H}) - \bm{I}) = n\sigma_{\varepsilon}^{2}\sum_{k = 1}^{n}(p_{k, t} - p_{0, t})\bm{v}_k\bm{v}_k^{\top},
\end{equation*}
where $\bm{v}_k$ is the eigenvector of $\bm{H}$ corresponding to eigenvalue $\lambda_k$,
$p_{k,t}=\exp(t\lambda_k)/\sum_{j=1}^{n}\exp(t\lambda_j)$ is the softmax weight assigned to
the $k$-th eigendirection of $\bm{H}$, and
$p_{0,t}=\{\sum_{j=1}^{n}\exp(t\lambda_j)\}^{-1}$. The above decomposition indicates that the eigenprojector $\bm{v}_k\bm{v}_k^{\top}$ is prioritized in $\var(\bm{u}_{t})$ according to the softmax value $p_{k, t}$.
As training progresses, the error variance $\Tr\{\var(\bm{\varepsilon}_{t})\} = n^{2}\sigma_{\varepsilon}^{2} / \sum_{k = 1}^{n}\exp(t\lambda_k)$ is monotonically decreasing in $t$, while the variance of the random effect, $\Tr\{\var(\bm{u}_{t})\} = n\sigma_{\varepsilon}^{2} - \Tr\{\var(\bm{\varepsilon}_{t})\}$, is monotonically increasing. Here, $\var(\bm{u}_t)$ represents the signal captured from the initial residual for prediction. The implied signal-to-noise ratio increases over time under the working random effect model. Specifically, the proportion of the variance explained by the random effect is given by 
$$\frac{\Tr\{\var(\bm{u}_{t})\}}{\sigma^2_{\text{Total}}} = 1 - \frac{n}{\sum_{k = 1}^{n}\exp(t\lambda_k)}.$$
Thus, training time $t$ can be viewed as a variance-allocation parameter. As $t$ increases, a larger share of the total marginal variation is attributed to the structured random effect rather than to residual noise.

In addition, the random effect perspective naturally motivates interpreting fixed-operator gradient-flow training as an implicit spectral regularization method. Note that $\bm{v}_{1},\dots, \bm{v}_{n}$ span $\mathbb{R}^n$, the predicted random effect $\widehat{\bm{u}}_{t}$ can be written as $\sum_{k = 1}^{n}\widehat{a}_{k,t}\bm{v}_{k}$, where the coefficients $\{\widehat{a}_{k,t}\}_{k=1}^n$ are characterized by the following proposition, which establishes a geometric interpretation of the gradient-flow training dynamics. This result also shows the converse side of the optimization--inference duality: the BLUP induced by the random-effects model is  the solution to a  spectral regularization problem.

\begin{proposition}
\label{prop: regularization}
For each $t> 0$, the $\rm BLUP$  $\widehat{\bmu}_t$ can be obtained by solving the following spectrally regularized optimization problem
$\widehat{\bmu}_t = \arg\min_{\bmu\in\mathbb{R}^n}\left\{\|\bmy - f_0(\bmX) - \bm{u}\|^2 + \sum_{k=1}^n(\bm{v}_k^\top \bm{u})^2/\{\exp(t\lambda_k) - 1\}\right\}$.
The predicted random effect $\widehat{\bm{u}}_{t}$ can be written as $\sum_{k = 1}^{n}\widehat{a}_{k,t}\bm{v}_{k}$, where
the coefficient vector $\widehat{\bma}_t=(\widehat{a}_{1,t},\dots,\widehat{a}_{n,t})^\top$ can be obtained by solving the following regularized optimization problem 
\begin{equation}
    \widehat{\bm{a}}_{t} = \arg\min_{\bm{a}\in \mathbb{R}^n} \left\{ \|\bmy - f_0(\bmX) - \sum_{k = 1}^{n}a_k\bm{v}_k\|^2 + \sum_{k=1}^n\frac{ a_k^{2}}{\exp(t\lambda_k) - 1} \right\} \label{eq: Regularization}
\end{equation}
with conventions that $1/0=\infty$ and $0\cdot\infty=0$.
In addition, we have $\widehat{a}_{k,t} = c_k \left(1-\exp(-t\lambda_k)\right)$ and the above objective function can be decomposed as
$\sum_{k=1}^n \mathcal{J}_k(a_k;t)$,
where $\mathcal{J}_k(a_k;t)=c_k^2-2c_ka_k+a_k^2/\{1-\exp(-t\lambda_k)\}$,
and $\mathcal{J}_k(\widehat{a}_{k,t};t) = c_k^2 \exp(-t\lambda_k)$.
\end{proposition}

When $\lambda_k=0$, the corresponding term is interpreted in the limiting sense, that is, the coefficient along that eigendirection is constrained to be zero, so that  $\widehat a_{k,t}=0$ and $\mathcal{J}_k(\widehat a_{k,t};t)=c_k^2$.
Equivalently, the coefficient-space optimization in \eqref{eq: Regularization} may be written only over eigendirections with $\lambda_k>0$.
We also note that, in some NTK settings, all eigendirections are active. For example, for feedforward networks with bias terms, at least two hidden layers, and a continuous, almost-everywhere differentiable, non-polynomial activation function, the infinite-width NTK matrix $\bmH_\infty$ is strictly positive definite for distinct inputs \citep{carvalho2025positivity}.

Proposition \ref{prop: regularization} shows that training induces an eigenspace-specific spectral regularization mechanism in the eigenbasis of the fixed operator $\bmH$. Specifically, the estimated coefficient along the $k$-th eigendirection satisfies $\widehat a_{k,t}=c_k\{1-\exp(-t\lambda_k)\}$, $1\leq k \leq n$. Because the factor $1-\exp(-t\lambda_k)$ increases monotonically with $\lambda_k$, directions associated with larger eigenvalues are learned more rapidly at an exponential rate and contribute more strongly to the fitted model, whereas directions associated with smaller eigenvalues are incorporated more gradually during training and zero-eigenvalue directions are left unactivated in training. Thus, early stopping acts as a coordinate-wise spectral regularization rule, controlling how quickly different eigenspaces of $\bmH$ are activated along the training trajectory. This eigenspace-specific filtering determines both how much each spectral component contributes to the fitted model and how much optimized spectral loss remains after training.

To characterize the spectral evolution of the training trajectory, we introduce the following notion of spectral loss decorrelation.
\begin{definition}[Spectral loss decorrelation]
Let $\mathcal{J}_{k,t}=\min_{a_k}\mathcal{J}_k(a_k;t)=\text{ER}_k^2(t)$ be the optimized spectral loss on the $k$-th eigenspace, equivalently, the squared eigen-residual. We say that spectral loss decorrelation holds at time $t$ if the empirical covariance between $\{\mathcal{J}_{k,t}\}_{k=1}^n$ and $\{\lambda_k\}_{k=1}^n$ is zero, i.e.,
$$
\frac{1}{n} \sum_{k=1}^n 
(\mathcal{J}_{k,t} - \bar{\mathcal{J}}_{t})(\lambda_k - \bar \lambda) = 0,
$$
where $\bar{\mathcal{J}}_{t} = n^{-1} \sum_{k=1}^n \mathcal{J}_{k,t}$ and $\bar \lambda = n^{-1} \sum_{k=1}^n \lambda_k$.
\end{definition}

Since $\mathcal{J}_{k,t} = c_k^2\exp(-t\lambda_k)$ according to Proposition \ref{prop: regularization}, the REML estimating equation in~\eqref{eq: REML EE} can be interpreted as selecting the training time $t$ such that the Empirical Spectral Covariance (ESC) $\Psi_n(t)$ between the optimized spectral loss $\{\mathcal{J}_{k,t}\}_{k = 1}^{n}$ and fixed-operator eigenvalues $\{\lambda_k\}_{k=1}^n$ is zero. In this sense, the REML-guided early stopping rule selects $t$ that achieves  \emph{spectral loss decorrelation}, as shown in Figure \ref{fig: spectral loss balance}(a).

\begin{figure}[!htbp]
    \centering
    \includegraphics[width=0.9\linewidth]{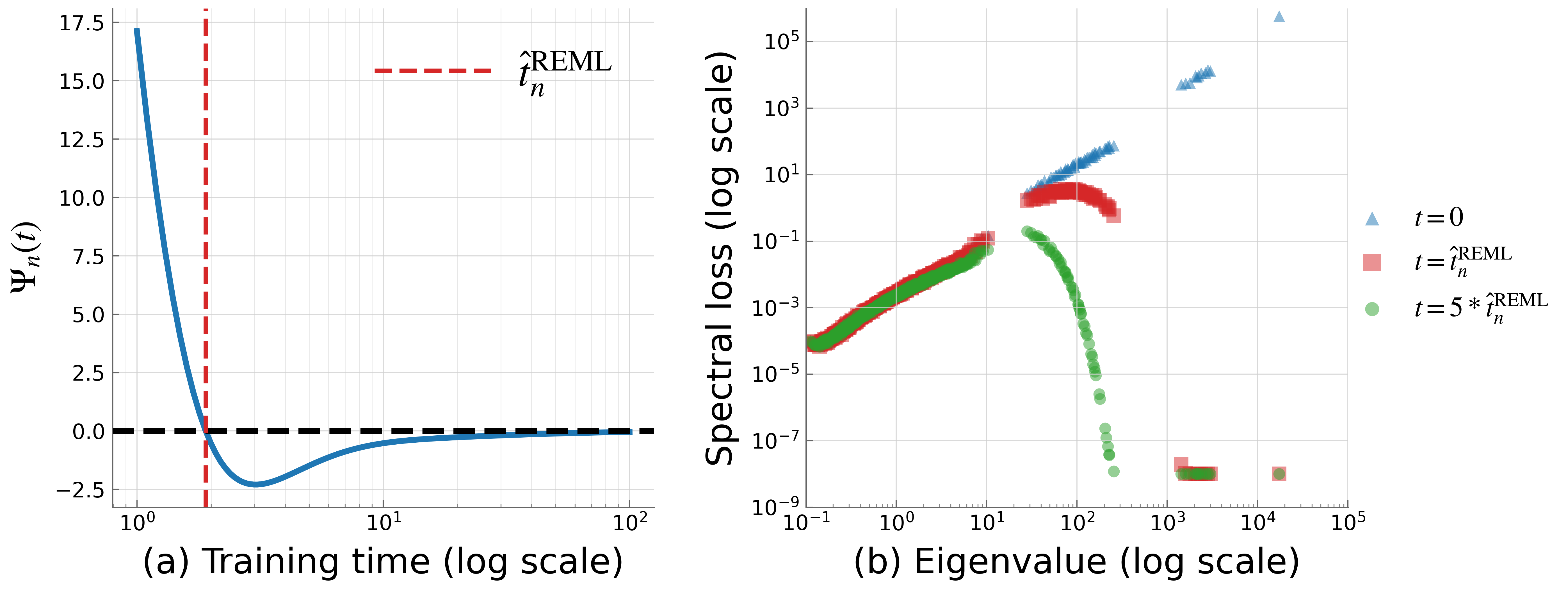}
    \caption{(a) Empirical spectral covariance (ESC) $\Psi_n(t)$ during the training process. The training time is displayed on the $\log_{10}$ scale. (b) Scatter plot of $\{(\lambda_k, \min_{a_k}\mathcal{J}_k(a_k; t))\}_{k=1}^n$  at three training times: $t=0$, $\widehat{t}_n^{\mathrm{REML}}$ and $5\times \widehat{t}_n^{\mathrm{REML}}$. The horizontal axis shows the eigenvalues of the fixed operator $\bmH$, and the vertical axis shows the corresponding optimized spectral losses; both axes are displayed on the $\log_{10}$ scale. For visualization, values on the vertical axis below $10^{-8}$ are truncated at $10^{-8}$.}
    \label{fig: spectral loss balance}
\end{figure}

Figure \ref{fig: spectral loss balance}(b) illustrates how training time reallocates the optimized spectral losses across eigendirections of the fixed operator $\bmH$. At $t<\widehat{t}_n^{\rm REML}$ (e.g. $t=0$), the empirical spectral covariance $\Psi_n(t)>0$, indicating that directions with larger eigenvalues still carry larger spectral losses; this corresponds to an underfitting regime, since even the dominant, more learnable eigendirections have not yet been sufficiently fitted. At $t=\widehat{t}_n^{\rm REML}$, we have $\Psi_n(\widehat{t}_n^{\rm REML})=0$, suggesting that the optimized spectral losses are balanced across eigendirections. This indicates a stopping time at which gradient flow has substantially learned the meaningful signal carried by the leading eigendirections, without yet overfitting the noisier small-eigenvalue directions. At $t>\widehat{t}_n^{\rm REML}$ (e.g., $t=5\times\widehat{t}_n^{\rm REML}$), the empirical spectral covariance $\Psi_n(t)<0$, meaning that the large-eigenvalue directions have already been fitted aggressively while the remaining loss is concentrated in the small-eigenvalue directions; this indicates an overfitting regime in which further training primarily fits noisy directions associated with small eigenvalues.

The spectral regularization perspective also yields a natural measure of model complexity. Let $\bmr_0=\bmy-f_0(\bmX)$. Since $\widehat{\bmu}_t = \left( \bm{I} - \exp\left( - t \bm{H} \right)\right) \bm{r}_0 $, the BLUP $\widehat{\bmu}_t$ is a linear smoother of $\bmr_0$ with smoother matrix $\bmS_t = \bm{I} - \exp\left( - t \bm{H} \right)$. Following \citet{buja1989linear}, 
we define the effective degrees of freedom (edf) of fixed-operator gradient flow at training time $t$ as
\begin{equation*}
    \mathrm{edf}(t) := \Tr(\bmS_t) = \Tr \left( \bm{I} - \exp\left( - t \bm{H} \right) \right) = \sum_{k=1}^n \{1-\exp(-t\lambda_k)\}.
\end{equation*}

This time-varying quantity provides a natural measure of model complexity along the training trajectory. Each fixed-operator eigendirection contributes $1-\exp(-t\lambda_k)$ to $\mathrm{edf}(t)$, so $\mathrm{edf}(t)$ summarizes the effective number of spectral directions activated by training time $t$. In particular, $\mathrm{edf}(0)=0$, and $\mathrm{edf}(t)$ increases monotonically with $t$, indicating that model complexity grows continuously during training. Thus, early stopping controls model complexity by regulating the extent to which the spectral directions of $\bmH$ are used.

\subsection{Prediction Optimality of REML-Guided Early Stopping}
\label{sec:prediction}
In this section, we study the prediction performance of the REML-guided early stopping time under the data-generating model in equation~\eqref{eqn:truemodel}. Let $\widehat{y}_{i,t}=\widehat f^{\bmH}_t(\bmx_i)$ denote the in-sample prediction for observation $i$ at training time $t\geq 0$.
Following the prediction-error framework for evaluating fitted predictors \citep{hastie2009elements}, we consider the following in-sample prediction risk
$$
\mathcal{E}_{n}(t;\bm H) = \frac{1}{n}\sum_{i = 1}^{n}\mathbb{E}[(y_{{\rm new}, i} - \widehat{y}_{i, t})^{2}],
$$
where $y_{{\rm new}, i} = f_{\star}(\bmx_{i}) + \varepsilon_{{\rm new}, i}$,  $\varepsilon_{{\rm new}, i}$ is independent of $\bmy$ and has the same distribution as $\varepsilon_{i}$ for $i = 1,\dots, n$. In the theoretical analysis, we allow $\var(\varepsilon_{i})$ to depend on $n$ and denote it by $\sigma_{\varepsilon, n}^{2}$. This allows the theoretical results to cover the different noise regimes, including those with $\lim_{n\to \infty}\sigma_{\varepsilon, n}^{2} > 0$ (standard-noise regime) and $\lim_{n\to \infty}\sigma_{\varepsilon, n}^{2} = 0$ (small-noise regime) \citep{carroll1984power,li2021multi}. For analytical convenience, we assume that the response and the initial prediction are centered such that $\mathbb{E}[n^{-1}\sum_{i = 1}^{n}y_{i}] = 0$ and $n^{-1}\sum_{i = 1}^{n}f_{0}(\bmx_{i}) = 0$. Next, we establish the oracle property of $\widehat{t}_{n}^{\rm REML}$. Note that
$$
  \bmy - \hat{\bmy}_{t} = \exp\left( - t \bm{H} \right) \left\{\bm{y} - f_0(\bmX) \right\},
$$
where $\hat{\bmy}_{t} = (\widehat{y}_{1, t}, \dots, \widehat{y}_{n, t})^{\top}$. In Section \ref{subsec: REML-guided time}, we show that the REML stopping time $\widehat{t}_{n}^{\rm REML}$ minimizes $Q(t)$, or equivalently, after an 
exponential transformation, minimizes:
$$
\begin{aligned}
    V_n(t) 
    = n^{-1}\exp\{n^{-1}Q(t)\}
    = \underbrace{n^{-1}(\bmy - \hat{\bmy}_{t})^\top \exp(t \bm{H}) (\bmy - \hat{\bmy}_{t})}_{\text{weighted training error}} \quad\times \underbrace{\exp(t\bar{\lambda}).}_{\text{spectral adjustment}} 
\end{aligned}
$$

In this decomposition of the REML criterion $V_n(t)$, the weighted training error term
$n^{-1}(\bmy - \hat{\bmy}_{t})^\top \exp(t \bm{H}) (\bmy - \hat{\bmy}_{t}) = n^{-1}(\bmy - f_{0}(\bmX))^\top \exp(-t \bm{H}) (\bmy - f_{0}(\bmX))$
decreases as $t$ increases and therefore favors longer training. By contrast, the spectral adjustment $\exp(t\bar{\lambda})$ increases with $t$ and counterbalances this decrease through the average eigenvalue of the training operator. The balance between these two factors determines the REML-guided stopping time. Moreover, the specific structure of $V_n(t)$ enables it to approximate the in-sample prediction risk near its minimizer.  Leveraging this observation, we establish the optimality of $\widehat{t}_n^{\mathrm{REML}}$ in terms of the in-sample prediction risk under the following conditions. Let
$b_k=\bmv_k^\top (f_{\star}(\bmX)-f_0(\bmX))$.
The following conditions formalize the requirements used in the risk comparison: attainability of the oracle risk, bounded initialization bias, sufficient spectral spread of the training operator, and concentration of the error terms.

\begin{condition}[Signal Concentration]\label{ass: learnable}
    $(n\sigma_{\varepsilon,n}^{2})^{-1}\sum_{\lambda_{k}\leq r_{n}\bar{\lambda}}b_{k}^{2} \to 0$ for some sequence $r_{n} \to \infty$ such that
$\log(1+\Gamma_n)/r_n\to0$ with
$\Gamma_n=\sum_{k=1}^n b_k^2/(n\sigma_{\varepsilon,n}^{2})$.
\end{condition}

\begin{condition}[Bounded projection]\label{ass: bounded projection}
    $n^{-1}\sum_{i=1}^{n}\{f_\star(\bmx_i) - f_0(\bmx_i)\}^2$ is uniformly bounded above.
\end{condition}

\begin{condition}[Existence of small eigenvalue]\label{ass: small eigenvalue}
    There is some $\alpha > 0$ and $\delta_{n} \to 0$ such that $n^{-1}\sum_{k = 1}^{n}1\{\lambda_k \leq \delta_{n}\bar{\lambda}\} \geq \alpha$.
\end{condition}
\begin{condition}[Sub-Gaussian error]\label{ass: subgaussian}
    The normalized error term $\sigma_{\varepsilon, n}^{-1}\varepsilon_{i}$ ($i = 1,\dots, n$) is sub-Gaussian with sub-Gaussian norm uniformly bounded above.
\end{condition}
Condition \ref{ass: learnable} assumes that the signal concentrates in the eigen-spaces with $\lambda_{k}> r_{n}\bar{\lambda}$. Condition \ref{ass: bounded projection} imposes a constraint on the initialization bias $f_{\star}(\bmX) - f_{0}(\bmX)$. Condition \ref{ass: small eigenvalue} requires that a non-negligible proportion of the eigenvalues $\{\lambda_k\}_{k=1}^n$ are substantially smaller than the average eigenvalue $\bar{\lambda}$, which is typically ensured when the eigenvalues decay sufficiently fast. For instance, suppose $\lambda_k / \lambda_1$ is of order $k^{-\zeta}$ for some $\zeta > 1$. Then $\bar{\lambda} / \lambda_1$ is of order $n^{-1}$, and the number of eigenvalues satisfying $\lambda_k / \lambda_1 \geq \delta_n  \bar{\lambda} / \lambda_1$ is of order $(n/\delta_n)^{1/\zeta}$. This quantity is $o(n)$ if one set $\delta_n = n^{-(\zeta - 1)/2}$. Consequently, $n^{-1}\sum_{k = 1}^{n}1\{\lambda_k \leq \delta_{n}\bar{\lambda}\} \to 1$ and Condition \ref{ass: small eigenvalue} holds for any $\alpha \in (0,1)$ under this choice of $\delta_n$.
Condition \ref{ass: subgaussian} is a standard assumption on the error term to establish concentration results \citep{wainwright2019high}.
We are now ready to state the following theorem.
\begin{theorem}[In-sample prediction optimality of REML-guided early stopping]
\label{thm: test error}
    Suppose $n^{a}\sigma_{\varepsilon,n}^{2} \to \infty$ for some $a \in (0, 1)$, $\sigma_{\varepsilon,n}^{2}$ is uniformly bounded above and $\lambda_n > 0$. Under the model in \eqref{eqn:truemodel} and Conditions \ref{ass: learnable}--\ref{ass: subgaussian}, we have
    $$
      \frac{\mathcal{E}_{n}(\widehat{t}_{n}^{\rm REML};\bm H)}{\inf_{t\geq 0}\mathcal{E}_{n}(t;\bm H)} \to 1,
    $$
    in probability as $n \to \infty$.
\end{theorem}

Theorem~\ref{thm: test error} establishes oracle optimality for fixed-design in-sample prediction. We next state the corresponding random-design result. Let
\begin{equation*}
\mathcal{E}_{n}^{*}(t;\bm H)
=
\mathbb{E}\Big[
\Big\{
y_{\rm new}-\widehat f^{\bmH}_t(\bmx_{\rm new})
\Big\}^{2}
\Big]
\end{equation*}
denote the out-of-sample prediction risk, where the expectation is taken over both the training data and an independent test observation and $\{\bm{x}_{i}\}_{i=1}^{n}$ is assumed to be i.i.d. Here 
$y_{\rm new}=f_{\star}(\bmx_{\rm new})+\varepsilon_{\rm new}$, and
$\bmx_{\rm new}$ and $\varepsilon_{\rm new}$ are independent of $\bmy$ and $\bmX$ and have the same
distributions as $\bmx_i$ and $\varepsilon_i$, respectively.

\begin{corollary}[Out-of-sample prediction optimality of REML-guided early stopping]
\label{cor:out_sample_optimality}
Assume that the conditions of Theorem~\ref{thm: test error} hold and
$\sqrt{n}\sigma_{\varepsilon,n}^{2}\to\infty$. In addition, assume that $h(\cdot,\cdot)$ is a bounded
Mercer kernel
\citep{mercer1909functions,scholkopf2002learning}, $\mathbb{E}\{f_0(\bmx)^4\}<\infty$, and the Condition \ref{ass: NTK spectral} in the Supplementary Materials
hold. Then, the REML-guided stopping time $\widehat{t}_{n}^{\rm REML}$ is asymptotically optimal
for out-of-sample prediction:
\begin{equation*}
\frac{
\mathcal{E}_{n}^{*}(\widehat{t}_{n}^{\rm REML};\bm H)
}{
\inf_{t\geq 0}\mathcal{E}_{n}^{*}(t;\bm H)
}
\to 1,
\end{equation*}
in probability as $n\to\infty$.
\end{corollary}

Proofs of Theorem~\ref{thm: test error} and Corollary~\ref{cor:out_sample_optimality}
are provided in the Supplementary Materials. Together, these results show that, for a specified fixed training operator $\bmH$, the REML early stopping time $\widehat{t}_{n}^{\rm REML}$ asymptotically attains the same prediction risk as the oracle stopping time along the corresponding fixed-operator gradient-flow trajectory, both for fixed-design in-sample prediction and, under additional regularity conditions, for random-design out-of-sample prediction.

Notably, although $\widehat{t}_{n}^{\rm REML}$ is motivated by the Gaussian likelihood,
Theorem~\ref{thm: test error} and Corollary~\ref{cor:out_sample_optimality} do not
require the errors $\varepsilon_i$ to be Gaussian. The results also allow the error
variance $\sigma_{\varepsilon,n}^{2}$ either to remain bounded away from zero, as in
standard regression settings, or to vanish as $n\to\infty$. The latter small-noise
regime is useful for describing settings in which the response is nearly deterministic
given the features \citep{li2021multi}. In this regime, the oracle risks in the denominators of
$\mathcal{E}_{n}(\widehat{t}_{n}^{\rm REML};\bmH)/\inf_{t\geq 0}\mathcal{E}_{n}(t;\bmH)$ and
$\mathcal{E}_{n}^{*}(\widehat{t}_{n}^{\rm REML};\bmH)/\inf_{t\geq 0}\mathcal{E}_{n}^{*}(t;\bmH)$
may converge to zero, while the ratios themselves still converge to one.

\begin{remark}
In the NTK settings, the results in Sections \ref{sec:test}--\ref{sec:prediction} apply by setting $\bmH=\bmH_\infty$, where $\bmH_\infty$ is the infinite-width NTK matrix on the training sample. The score test, empirical spectral covariance, effective degrees of freedom, and REML-guided stopping time are then defined with respect to this NTK operator. Thus, NTK gradient flow is an important special case of the fixed-operator random-effects inference framework developed above, and it provides the basis for the numerical experiments and UK Biobank application below.
\end{remark}

\section{Numerical Experiments}
\label{sec:numerical}

\subsection{Type I Error and Power of the Proposed Score Test}

We evaluate the performance of the proposed score test in terms of both Type I error control and statistical power. For each training sample size $n_{\text{train}}\in\{100, 200, 300, 400, 500\}$, we generate $1{,}000$ Monte Carlo replications. In each replication, we generate the features $\bmX\in \mathbb{R}^{n_\text{train}\times p}$ with $p=10$, where each feature is independently drawn from the standard normal distribution $N(0,1)$. Under the null hypothesis, we generate the response $y=\varepsilon\sim N(0,0.25)$. Under the alternative hypothesis, we generate $y=f_\star(\bmx)+\varepsilon$ with $f_\star(\bmx)=0.25 x_1+0.15 x_2^2 + 0.1 x_3x_4 + 0.2 \sin(x_5) + 0.15 \cos(x_6)\sin(x_7) + 0.05 x_8 x_9 x_{10}$ and $\varepsilon\sim N(0, 0.25)$. The significance level is set to $\alpha=0.05$. We then perform the score test by computing the empirical NTK matrix using a two-layer neural network with width $w=500$. We report the empirical Type I error rates and power in Supplementary Table \ref{tab: type1_power}. Overall, the proposed projected NTK-REML score test maintains nominal Type I error control across all simulated training sample sizes, and its power approaches one as the sample size increases.

\subsection{REML Provides a Principled Early-stopping Time}
\label{sec: simu early stop}

In this section, we conduct numerical experiments to evaluate the early-stopping time estimated by REML across three cases. In all cases, the predictor $\bmx=(x_1,\dots,x_{10})^\top\in\mathbb{R}^{10}$ is generated with independent features $x_j\sim N(0,1)$. The training and test sample sizes are $n_{\mathrm{train}}=1{,}000$ and $n_{\mathrm{test}}=100$, respectively. The response is generated as $y = f_\star(\bmx) + \varepsilon$ with $\varepsilon\sim N (0,0.25)$. We consider the following three ground truth functions: (1) Case 1: $f_\star(\bmx) = 0.1 x_1 + 0.16\tanh(x_2) + 0.2\sin(x_3) + 0.12x_4 + 0.06x_5^2 + 0.01e^{x_6} + 0.2\cos(x_7) + 0.1|x_8| + 0.08x_9 + 0.14\sin(x_{10})$; (2) Case 2: $f_\star(\bmx) =  0.25x_1 + 0.15x_2^2 + 0.1x_3x_4 + 0.2\sin(x_5) + 0.15\cos(x_6)\sin(x_7) + 0.05x_8x_9x_{10}$; (3) Case 3: $f_\star(\bmx) = 2\cos\!\big(\sin(x_1x_2+x_3x_4x_5) + \sin(x_3x_4x_5) + 2(\sin(x_6)\sin(x_7)+\sin(x_8)\sin(x_9)\sin(x_{10})\big)$. 

For each case, we train a fully connected neural network with two hidden layers of width $w=1{,}000$ and ReLU activation using full-batch gradient descent with learning rate $\eta=10^{-2}$. We compare the training dynamics of gradient descent with the corresponding NTK gradient flow solution, and evaluate performance using mean squared error (MSE) on both training and test datasets. To quantify performance, we consider two metrics: (1) test error; (2) computational time. The results are shown in Figure~\ref{fig: simu REML time}.

\begin{figure}[!htbp]
    \centering
    \includegraphics[width=1\linewidth]{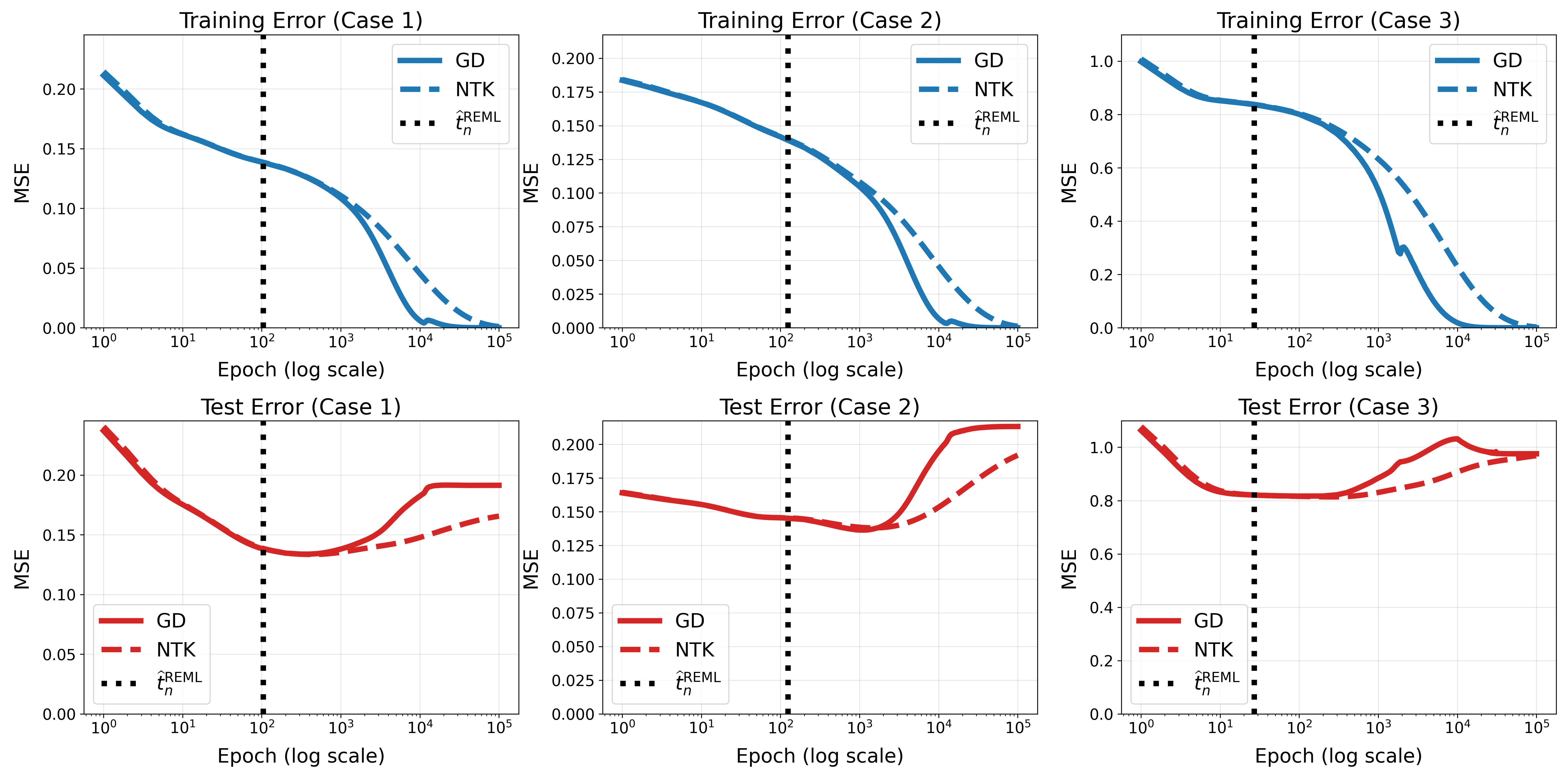}
    \caption{Comparison of gradient descent and NTK gradient flow in numerical experiments. The vertical dotted line $\widehat{t}_n^{\rm REML}$ represents the  REML early-stopping time.}
    \label{fig: simu REML time}
\end{figure}

Figure~\ref{fig: simu REML time} shows that,  in all three cases, the  REML-guided early stopping time occurs near the onset of the test-error plateau, where further training yields diminishing improvements. In addition, the gradient descent trajectories closely track the corresponding NTK gradient flow trajectories for both training and test errors in all three cases, especially before $10^3$ epochs. In Case 1, REML selects $\widehat{t}_n^{\mathrm{REML}}=104$ as the early stopping time, whereas the oracle stopping time is $3.83$ times longer. In the more complex Cases 2 and 3, the REML early stopping times are $\widehat{t}_n^{\mathrm{REML}}=124$ and $27$ epochs, respectively, while the oracle stopping epochs are approximately $8.90$ and $3.74$ times longer, respectively. The edf values for Cases 1--3 are $26.13$, $28.31$, and $9.53$, respectively. To characterize the dominant NTK eigenspaces, we select the leading NTK eigenvalues using a cumulative-eigenvalue criterion, as shown in Supplementary Figure \ref{fig: scree_plot}. Across all three cases, this criterion identifies a small set of leading eigenvalues, reflecting the rapid decay of the NTK spectrum and the concentration of training dynamics in low-dimensional spectral components.

We also compare the REML stopping time with the validation-based stopping time by splitting the $1{,}000$ training samples into training and validation subsets with a $2{:}1$ ratio \citep{prechelt1998early}. Validation-based early stopping requires evaluating validation loss over a checkpoint grid, whereas REML estimates the stopping time directly from the fixed training operator constructed at initialization. Compared with validation-based early stopping, REML achieves comparable test errors across all three cases, with test errors of $0.1384$, $0.1456$, and $0.8207$ compared with $0.1391$, $0.1494$, and $0.8274$ for validation. REML also substantially reduces computational time, requiring only $79.45$, $82.49$, and $66.48$ seconds, compared with $9{,}812.24$, $9{,}938.66$, and $9{,}942.61$ seconds for validation.

We further conduct numerical experiments using a Regression Transformer model \citep{born2023regression} and obtain similar results: the REML-guided early-stopping criterion achieves comparable test error with substantially shorter computational time than the validation-based method, as shown in Supplementary Section \ref{sec: supp simu}. All computations were performed on a Dell Pro Max 16 workstation equipped with an Intel Core Ultra 9 285H CPU, Intel Arc integrated GPU, and 32 GB RAM.

\section{Application to UK Biobank Proteomics Data}
\label{sec:UKB}
In this section, we evaluate the proposed random-effects model-based testing procedure and the REML-guided early stopping rule using data from the UK Biobank Pharma Proteomics Project \citep{sun2023plasma}. We analyzed protein measurements and trait values from $42{,}054$ independent individuals. We trained neural networks to predict four quantitative clinical traits from protein measurements: low-density lipoprotein (LDL), hemoglobin A1c (HbA1c), systolic blood pressure (SBP), and body mass index (BMI).

For each clinical trait, we selected the 20 proteins with the largest absolute marginal correlations with the trait based on a subsample of $2{,}000$. To avoid double-dipping, this screening subsample was excluded from all subsequent analyses. This marginal-screening step is a dimension-reduction procedure commonly used in high-dimensional prediction settings \citep{fan2008sure}. The selected proteins were then used as predictors in the neural network models. The resulting data were randomly split into training and test sets with a ratio of $2{:}1$. We used a fully connected neural network with one hidden layer of width $w = 4{,}000$ and ReLU activation.
We first assessed whether training a neural network provided predictive benefit using the proposed random-effects model-based test. To reduce computational cost, the test was conducted on a random subsample of size $n_{\rm sub}=4{,}000$. The resulting $p$-values for all four traits are effectively zero, indicating strong evidence that neural network training improves predictive performance in all cases.

\begin{figure}[!htbp] \centering \includegraphics[width=1.0\linewidth]{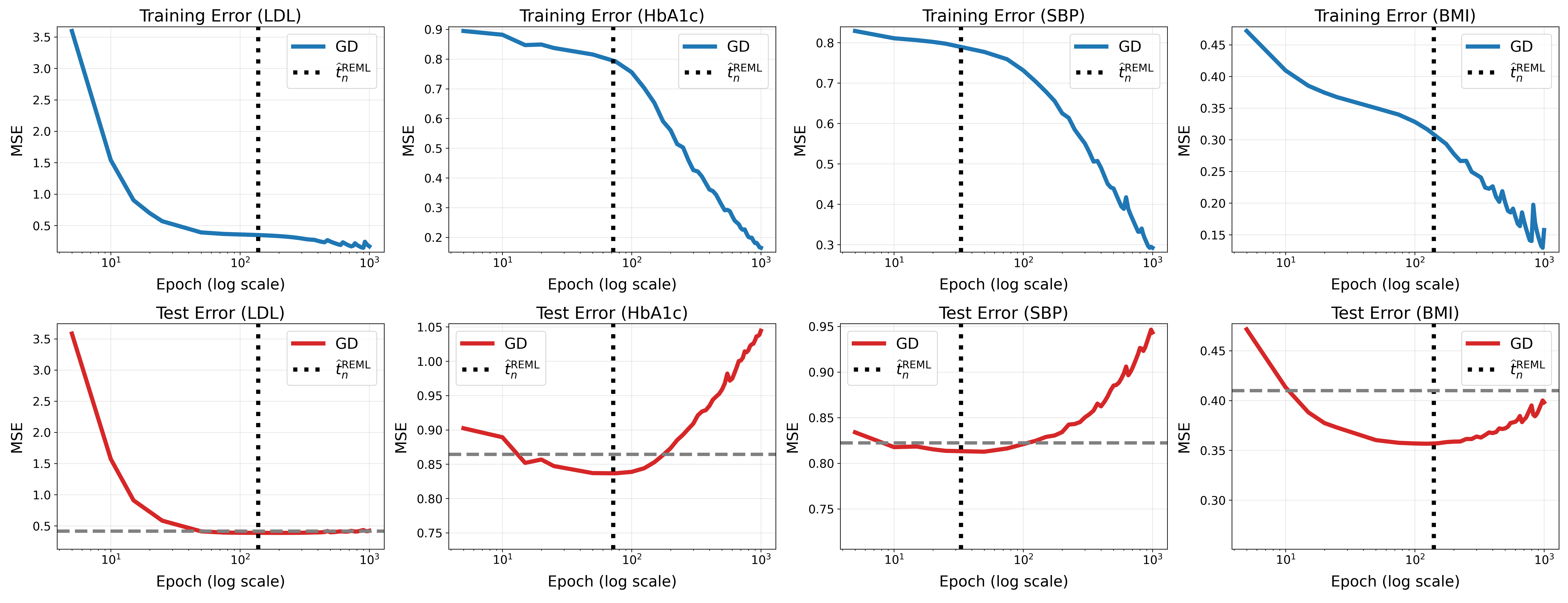} \caption{Training and test MSE curves for LDL, HbA1c, SBP, and BMI. The black dotted line indicates the REML-guided early stopping time. The grey dashed line indicates the test error of the linear regression as a benchmark.} \label{fig: UKBB-PPP REML time} \end{figure}

The neural network was trained using full-batch gradient descent with learning rate $\eta=10^{-2}$. The training and test errors are summarized in Figure~\ref{fig: UKBB-PPP REML time}. The errors reported in Figure~\ref{fig: UKBB-PPP REML time} are the MSEs normalized by the variance of the corresponding trait. For all four traits, the training loss decreases monotonically as training proceeds, while the test loss first decreases and then increases, reflecting overfitting. The REML-guided stopping time suggests training the neural networks for $138$, $72$, $33$,  and $140$ gradient descent iterations when modeling LDL, HbA1c, SBP, and BMI, respectively. The edfs, which measure the model complexity at the REML-guided early stopping time, are $190.56$, $127.23$, $77.97$, and $169.13$, respectively.
In all four cases, the REML-guided early stopping rule achieves near-optimal prediction performance on test data along the training trajectory.

REML achieves test MSEs comparable to validation-based early stopping for all four traits with $20\%$ validation data (using $1{,}000$ training epochs and selecting the training time with the smallest MSE on the validation data). Specifically, REML-guided early stopping achieves test MSEs of $0.3878$, $0.8386$, $0.8136$, and $0.3581$ for LDL, HbA1c, SBP, and BMI, respectively. For these traits, the test MSEs of validation-based early stopping are $0.3887$, $0.8395$, $0.8147$, and $0.3596$, respectively. REML also substantially reduces computational time, requiring $2{,}604$, $2{,}033$, $1{,}690$ and $2{,}511$ seconds, compared with $10{,}116$, $10{,}043$, $10{,}057$, and $8{,}437$ seconds for validation-based early stopping.

\section{Discussion}
\label{sec:discuss}

This paper develops an optimization--inference duality for fixed-operator squared-error gradient flow. When the fitted values evolve through a fixed positive semidefinite training operator $\bmH$, the gradient-flow trajectory is not only an optimization path but also a random-effects inference path: the fitted predictor is a BLUP, and training time acts as a variance-component parameter governing signal extraction and model complexity. The NTK regime, obtained by taking $\bmH=\bmH_\infty$, provides a canonical deep-learning example, but the random-effects representation itself is a general property of fixed-operator gradient flow. A distinctive feature of this formulation is that the learned predictor is characterized directly in prediction space. Even in linear regression settings, the fitted-value trajectory can be written in terms of the training operator $\bmH=\bmX\bmX^\top$ and the stopping time $t$, without explicitly estimating the regression coefficient vector. Thus, the inferential target is not the underlying model parameter, but the training-time variance component governing movement along the operator-induced prediction trajectory.

This viewpoint turns two basic training decisions into inferential problems. Whether to train can be assessed by a variance-component score test for signal beyond initialization. Conditional on evidence that training is useful, the stopping time can be estimated by REML. Thus, early stopping is no longer only a validation-based or heuristic tuning device~\citep{prechelt1998early,yao2007early,goodfellow2016deep}; it becomes estimation of a variance-component parameter controlling the complexity of the trained predictor. Practically, this avoids reserving part of the sample for validation and reduces reliance on repeated checkpoint evaluation, which can be useful in moderate-sample or low-signal settings.

The random-effects representation also gives an interpretable spectral characterization of training dynamics. Along the fixed-operator gradient-flow path, residuals decay exponentially in the eigendirections of $\bmH$: larger-eigenvalue directions are fitted earlier, whereas smaller-eigenvalue directions are incorporated later. The REML estimating equation selects the training time at which the optimized spectral losses become empirically decorrelated from the eigenvalues of $\bmH$. This yields a natural stopping principle: before this point, dominant eigendirections remain underfit; after this point, further training increasingly targets lower-eigenvalue directions that are more likely to contain noise. The same structure gives an effective degrees-of-freedom measure, $\mathrm{tr}\{\bmI-\exp(-t\bmH)\}$, linking training duration, spectral regularization, and model complexity.

This dynamic spectral view differs from classical kernel ridge regression and kernel-machine mixed models~\citep{liu2007semiparametric}. In those settings, regularization is imposed statically through an explicit ridge, smoothing, or variance-ratio parameter. Here, regularization is generated dynamically by the training trajectory. The covariance $\sigma_{\varepsilon,t}^2\{\exp(t\bmH)-\bmI\}$ is therefore not an arbitrary working choice, but the covariance needed for the BLUP to reproduce the gradient flow trajectory. From this perspective, over-parameterization alone does not determine model complexity; realized complexity depends on how long training is allowed to activate spectral directions of $\bmH$.

Although the REML criterion is motivated by a Gaussian working random-effects model, the prediction guarantees do not require Gaussian errors. Under regularity conditions, the REML-guided stopping time  asymptotically achieves the same fixed-design in-sample prediction risk as the oracle stopping time, and the guarantee extends to random-design out-of-sample prediction under additional regularity conditions. Thus, the random-effects model serves as a useful working inferential device whose induced stopping rule remains prediction-optimal under broader data-generating mechanisms.

Several limitations suggest directions for future work. First, the exact equivalence relies on a fixed training operator, continuous-time dynamics, and squared-error loss. In deep learning, this corresponds most directly to fixed-kernel or linearized regimes~\citep{jacot2018neural,lee2019wide,malladi2023kernel,afzal2025linearization,li2025efficient,wang2025selecting}. Finite-width networks trained by discrete-time or stochastic optimization may depart from this setting, especially when representation learning causes the tangent operator to evolve during training~\citep{Chizat2018,woodworth2020kernel,yang2021feature,radhakrishnan2024mechanism}. Extending the framework to time-varying operators, feature-learning regimes, and momentum-based dynamics is therefore an important next step~\citep{mei2019mean,Chizat2018,montanari2025dynamical,su2016differential}. Second, extending the approach beyond squared-error loss to classification or other generalized response models will require new analysis, since the likelihood structure and induced covariance dynamics differ from the Gaussian case~\citep{radhakrishnan2023wide}.

Computation is another important issue. The theory is formulated for a fixed positive semidefinite operator $\bmH$ on the training sample, but forming and decomposing a dense sample-by-sample operator can be demanding in large cohorts or modern-AI applications. Scalable implementations may use sparse or low-rank approximations, randomized eigensolvers, kernel sketches, blockwise computation, or matrix--vector products that avoid explicitly storing the full operator. Related ideas have been successful in large-scale mixed-model analyses in statistical genetics~\citep{yang2011gcta,loh2015bolt,jiang2019fastgwa,mbatchou2021regenie} and in scalable kernel approximation~\citep{drineas2005nystrom,gittens2013revisiting,si2016computationally,si2017memory,li2016fast}. Quantifying how such operator approximations affect the score test, the REML-guided stopping time, the effective degrees of freedom, and the prediction guarantees is an important topic for future work.

More broadly, the proposed framework suggests a statistical research agenda for inference on optimization dynamics. In the present paper, a specified model class, training regime, fixed operator $\bm H$, and initial condition $f_0(\bm X)$ determine a prediction-space trajectory, and the inferential target is the stopping time along that trajectory. Other modeling and training choices, such as architecture, initialization, learning rate, regularization, batch size, data augmentation, and optimizer design, may change the operator or the resulting training dynamics. Extending the random-effects viewpoint to these settings would require identifying the appropriate dynamical representation and developing inferential tools for how such choices affect variance allocation, spectral shrinkage, and effective model complexity.

\section*{Code and Data Availability}

All code used in this study is publicly available in the accompanying Python package \texttt{GF-REML} at \href{https://github.com/MinhaoYaooo/GF-REML}{https://github.com/MinhaoYaooo/GF-REML}.
The data analyzed in this work were obtained from the UK Biobank under approved application and are available through the UK Biobank Research Analysis Platform (\href{https://www.ukbiobank.ac.uk/use-our-data/research-analysis-platform}{https://www.ukbiobank.ac.uk/use-our-data/research-analysis-platform}), subject to UK Biobank data access policies.

\section*{Author Contributions}

Z.L. conceived the project, originated the central idea, designed and supervised the research, and led the overall project. Z.L., M.Y., and R.W. developed the methodology. M.Y. developed the code and software and performed the simulation studies. R.W. performed the data analysis. M.Y., R.W., and Z.L. drafted and revised the manuscript. L.L. and X.L. contributed to methodological discussion and manuscript editing. All listed authors reviewed and approved the final manuscript. M.Y. and R.W. contributed equally to this work.

\section*{Acknowledgement}
This research was conducted using the UK Biobank Resource under Application Number 52008.

\bibliography{DNN}
\bibliographystyle{apalike}

\clearpage
\appendix

\renewcommand{\theequation}{S\arabic{equation}}
\setcounter{equation}{0}
\renewcommand{\theHequation}{supp.\arabic{equation}}

\renewcommand{\thetheorem}{S\arabic{theorem}}
\setcounter{theorem}{0}
\renewcommand{\theHtheorem}{supp.\arabic{theorem}}

\renewcommand{\thecondition}{S\arabic{condition}}
\setcounter{condition}{0}
\renewcommand{\theHcondition}{supp.\arabic{condition}}

\renewcommand{\theremark}{S\arabic{remark}}
\setcounter{remark}{0}
\renewcommand{\theHremark}{supp.\arabic{remark}}

\renewcommand{\thelemma}{S\arabic{lemma}}
\setcounter{lemma}{0}
\renewcommand{\theHlemma}{supp.\arabic{lemma}}

\renewcommand{\thesection}{S\arabic{section}}
\setcounter{section}{0}
\renewcommand{\theHsection}{supp.\arabic{section}}

\renewcommand{\thefigure}{S\arabic{figure}}
\setcounter{figure}{0}
\renewcommand{\theHfigure}{supp.\arabic{figure}}

\renewcommand{\thetable}{S\arabic{table}}
\setcounter{table}{0}
\renewcommand{\theHtable}{supp.\arabic{table}}

\begin{center}
{\Large \bfseries Supplementary Materials}
\end{center}

\section{Proof of Theorem \ref{thm: equivalence}}

Throughout the proof, we use $c$ and $C$ to denote generic positive constants whose values may differ from place to place.

Let $\bm r_0=\bmy-f_0(\bmX)$. By the fixed-operator gradient-flow solution in \eqref{eq: fixed operator solution}, we have
\begin{equation*}
\widehat{f}_t^{\bmH}(\bmX) = f_0(\bmX) + \{\bm{I}-\exp(-t\bmH)\}\bm r_0 .
\end{equation*}
We now show that the BLUP under the working random-effects model \eqref{eq: model RE}
gives the same expression. Under \eqref{eq: model RE}, we have
\begin{equation*}
\bm r_0=\bm{u}_t+\bm{\varepsilon}_t,
\end{equation*}
where
\begin{equation*}
\var(\bm{u}_t)
=
\sigma_{\varepsilon,t}^2\{\exp(t\bmH)-\bm{I}\},
\qquad
\var(\bm{\varepsilon}_t)
=
\sigma_{\varepsilon,t}^2\bm{I},
\qquad
\bm{u}_t\perp\bm{\varepsilon}_t .
\end{equation*}
Therefore,
\begin{equation*}
\var(\bm r_0)
=
\var(\bm{u}_t)+\var(\bm{\varepsilon}_t)
=
\sigma_{\varepsilon,t}^2\exp(t\bmH).
\end{equation*}
Although $\exp(t\bmH)-\bm{I}$ may be singular when $\bmH$ is only positive semidefinite,
$\exp(t\bmH)$ is positive definite for every $t\geq 0$, since the eigenvalues of
$\exp(t\bmH)$ are $\exp(t\lambda_k)>0$. Hence the BLUP of $\bm{u}_t$ is
\begin{equation*}
\begin{aligned}
\widehat{\bm{u}}_t
&= \Cov(\bm{u}_t,\bm r_0)\{\var(\bm r_0)\}^{-1}\bm r_0  \\
&= \sigma_{\varepsilon,t}^2\{\exp(t\bmH)-\bm{I}\}
\left\{\sigma_{\varepsilon,t}^2\exp(t\bmH)\right\}^{-1}
\bm r_0  \\
&= \{\bm{I}-\exp(-t\bmH)\}\bm r_0 .
\end{aligned}
\end{equation*}
Consequently,
\begin{equation*}
\begin{aligned}
f_0(\bmX)+\widehat{\bm{u}}_t
&=
f_0(\bmX)
+
\{\bm{I}-\exp(-t\bmH)\}
\{\bmy-f_0(\bmX)\}  \\
&=
\widehat{f}_t^{\bmH}(\bmX),
\end{aligned}
\end{equation*}
which proves the equivalence. The above proof is valid for any positive $\sigma_{\varepsilon,t}^2$ and we adopt $\sigma_{\varepsilon,t}^{2}
=
\gamma_t^{-1}\sigma_{\varepsilon}^{2}$ to ensure the initial risk $n^{-1}\mathbb{E}[\|\bmy-f_0(\bmX)\|^{2}]$ does not depend on $t$ under the random effect model.

\begin{remark}[Time-varying training operators]
\label{remark: time-varying}
Theorem \ref{thm: equivalence} is stated for a fixed operator $\bmH$, but the time-varying formulation in \eqref{eq:timevaryingOperator} in the main text suggests a natural cumulative-operator extension. 
Suppose the time-varying positive semidefinite training operators are pairwise commuting over time, that is, $\bmH(s)\bmH(r)=\bmH(r)\bmH(s)$ for all $r,s\in[0,t]$.
The gradient-flow solution becomes 
$\widehat{f}_t(\bm{X}) = f_0(\bm{X}) + \bigl(\bm{I} - \boldsymbol{\mathcal{T}}(t)\bigr)\bigl\{\bm{y} - f_0(\bm{X})\bigr\}$, 
where $\boldsymbol{\mathcal{T}}(t) = \!\exp\!\bigl\{-\int_0^t \bmH(s) \, ds\bigr\}$. The random-effects representation above holds with $\exp\{t\bmH\}$ replaced by $ \!\exp\!\bigl\{\int_0^t \bmH(s) \, ds\bigr\}$. 
\end{remark}

\section{Proof of Proposition \ref{prop: existence uniqueness empirical REML}}

First, we have proven that $Q(t)$ is convex on $[0,\infty)$. Because $\varepsilon_i$ is a continuous random variable for $i = 1,\dots,n$, $c_k \neq 0$ for all $k = 1,\dots,n$ with probability one. Under condition (ii) in Proposition \ref{prop: existence uniqueness empirical REML}, the support of the probability weights $p_k(t)$ contains at least two distinct eigenvalues for every finite $t\ge 0$. Hence
\begin{equation*}
Q''(t)=n\,\var_{p_t}(\lambda)>0,\qquad t\ge 0,
\label{eq: supp strict convexity}
\end{equation*}
so $Q(t)$ is strictly convex on $[0,\infty)$. Therefore $Q'(t)$ is strictly increasing, and any solution to $Q'(t)=0$ is unique.

It remains to show the existence of a root. Condition (i) gives $Q'(0)<0$. Next, as $t\to\infty$, for every $k\in\{1,\cdots,n\}$,
\begin{equation*}
p_k(t)=\frac{c_k^2\exp(-t\lambda_k)}{\sum_{j=1}^n c_j^2\exp(-t\lambda_j)}
\label{eq: supp pt limit}
\end{equation*}
concentrates on the minimum eigenvalue $\lambda_{n}$. Consequently,
\begin{equation*}
\sum_{k=1}^n p_k(t)\lambda_k \longrightarrow \lambda_{n},
\qquad t\to\infty.
\label{eq: supp mean limit}
\end{equation*}
Therefore,
\begin{equation*}
Q'(t)\longrightarrow n(\bar{\lambda}-\lambda_{n}),
\qquad t\to\infty.
\label{eq: supp Qprime limit}
\end{equation*}
Condition (ii) implies that this limit is strictly positive. Since $Q'(t)$ is continuous on $[0,\infty)$, there exists at least one $t\in(0,\infty)$ such that $Q'(t)=0$. By strict convexity, this root is unique. The equivalence between $Q'(t)=0$ follows from \eqref{eq: REML EE}. Since $Q(t)$ is strictly convex, this unique critical point is also the unique minimizer of $Q(t)$ over $[0,\infty)$.

\begin{remark}
Condition (i) requires that the REML profile objective initially decreases away from $t=0$. This is equivalent to $\sum_{k=1}^n c_k^2\lambda_k/\sum_{j=1}^n c_j^2>\bar{\lambda}$ which can be further written as the covariance form
$$
\frac{1}{n}\sum_{k=1}^{n}(c_{k}^{2} - \bar{c^{2}})(\lambda_k - \bar{\lambda}) > 0
$$
with $\bar{c^{2}} = n^{-1}\sum_{k=1}^{n}c_{k}^{2}$.
Condition (i) can be satisfied if the projection coefficient $c_k^2$ tends to be larger for large eigenvalues. Condition (ii) ensures that the eigenspaces include at least one eigenvalue below $\bar{\lambda}$, so that $Q'(t)$ eventually becomes positive as $t\to\infty$. In addition, condition (ii) guarantees strict convexity of $Q(t)$, and hence uniqueness of the empirical REML-guided stopping time. The spectral patterns underlying conditions (i) and (ii) can be empirically assessed from the data, as illustrated in Figure \ref{fig: spectral loss balance}(b) in the main text. 
\end{remark}

\section{Proof of Proposition \ref{prop: regularization}}
Consider the proposed regularization objective function:
\begin{equation}
    \mathcal{J}(\bm{u}) = \|\bmy - f_0(\bmX) - \bm{u}\|^2 + \sum_{k=1}^n \rho_k(t) (\bm{v}_k^\top \bm{u})^2, \label{eq: supp obj}
\end{equation}
where $\rho_k(t) = \frac{\exp(-t\lambda_k)}{1 - \exp(-t\lambda_k)}$. Let $c_k = \bm{v}_k^\top (\bmy - f_0(\bmX))$ and $a_k = \bm{v}_k^\top \bm{u}$ be the projections of the initial residual and the optimization variable onto the eigenvector $\bm{v}_k$, respectively. Then, $\bm{u} = \sum_{k=1}^{n}a_k\bm{v}_k$. Since the eigenvectors $\{\bm{v}_k\}_{k=1}^n$ form an orthonormal basis, we can rewrite the Euclidean norm in the eigenbasis:
\begin{equation*}
    \|\bmy - f_0(\bmX) - \bm{u}\|^2 = \sum_{k=1}^n (c_k - a_k)^2.
\end{equation*}
Substituting this into \eqref{eq: supp obj}, the objective function decouples into a sum of independent component-wise problems under the reparametrization $\bm{u} = \sum_{k=1}^{n}a_k\bm{v}_k$:
\begin{equation*}
    \mathcal{J}(\bm{u}) = \sum_{k=1}^n \mathcal{J}_k(a_k;t) = \sum_{k=1}^n \left[ (c_k - a_k)^2 + \rho_k(t) a_k^2 \right].
\end{equation*}
To find the minimizer $\widehat{\bm{u}}_t$, we take the partial derivative with respect to each component $a_k$ and set it to zero:
\begin{equation*}
    \frac{\partial \mathcal{J}}{\partial a_k} = -2(c_k - a_k) + 2\rho_k(t) a_k = 0 \implies (1 + \rho_k(t))a_k = c_k.
\end{equation*}
Solving for $a_k$, we obtain:
\begin{equation*}
    \widehat{a}_{k,t} = \frac{c_k}{1 + \rho_k(t)}.
\end{equation*}
Substituting the definition of $\rho_k(t)$:
\begin{equation*}
    1 + \rho_k(t) = 1 + \frac{\exp(-t\lambda_k)}{1 - \exp(-t\lambda_k)} = \frac{1 - \exp(-t\lambda_k) + \exp(-t\lambda_k)}{1 - \exp(-t\lambda_k)} = \frac{1}{1 - \exp(-t\lambda_k)}.
\end{equation*}
Therefore, the optimal solution component is:
\begin{equation*}
    \widehat{a}_{k,t} = c_k (1 - \exp(-t\lambda_k)).
\end{equation*}
This is exactly the projection of the fixed-operator gradient flow solution $\widehat f^{\bmH}_t(\bmX) - f_0(\bmX)$ onto the $k$-th eigenvector. Thus, $\widehat{\bm{u}}_t=\sum_{k = 1}^{n}\widehat{a}_{k,t}\bm{v}_{k}$ corresponds precisely to the fixed-operator gradient flow solution at time $t$.

Then, we evaluate the objective function $\mathcal{J}(\bm{u})$ at the optimal solution $\widehat{\bm{u}}_t$. Using the relation $c_k - \widehat{a}_{k,t} = \rho_k(t)\widehat{a}_{k,t}$ derived from the first-order condition, the term for the $k$-th component becomes:
\begin{align*}
   \mathcal{J}_k(\widehat{a}_{k,t};t) &= (c_k - \widehat{a}_{k,t})^2 + \rho_k(t) \widehat{a}_{k,t}^2 \nonumber \\
    &= (\rho_k(t) \widehat{a}_{k,t})^2 + \rho_k(t) \widehat{a}_{k,t}^2 \nonumber \\
    &= \rho_k(t) (1 + \rho_k(t)) \widehat{a}_{k,t}^2.
\end{align*}
Substituting $\widehat{a}_{k,t} = \frac{c_k}{1 + \rho_k(t)}$, we have:
\begin{equation*}
    \mathcal{J}_k(\widehat{a}_{k,t};t) = \rho_k(t) (1 + \rho_k(t)) \left( \frac{c_k}{1 + \rho_k(t)} \right)^2 = c_k^2 \frac{\rho_k(t)}{1 + \rho_k(t)}.
\end{equation*}
Recalling that $\frac{\rho_k(t)}{1+\rho_k(t)} = \exp(-t\lambda_k)$, the component-wise minimized loss simplifies to:
\begin{equation*}
    \mathcal{J}_k(\widehat{a}_{k,t};t) = c_k^2 \exp(-t\lambda_k).
\end{equation*}

\section{Proof of Theorem \ref{thm: test error}}
In this and the subsequent sections, we omit the $\bm{H}$ in $\mathcal{E}_{n}$ and $\mathcal{E}_{n}^{*}$ for simplicity of notation.
Note that $\widehat{t}_{n}^{\rm REML}$ is the minimizer of
$$
\begin{aligned}
    V_n(t) 
    & = n^{-1}\exp\{n^{-1}Q(t)\}\\
    & = n^{-1}(\bmy - f_{0}(\bmX))^\top \exp(-t \bm{H}) (\bmy - f_{0}(\bmX))\exp(t\bar{\lambda}).
\end{aligned}
$$
We introduce three lemmas that are used in the proof of Theorem \ref{thm: test error}, whose proofs will be provided later.
\begin{samepage}
\begin{lemma}\label{lemma: t_star control}
    Under Condition \ref{ass: small eigenvalue}, we have $\liminf_{n\to\infty}n^{-1}\sum_{k = 1}^{n}\exp(-\bar{\lambda}^{-1}\lambda_k+1) \geq \alpha e+(1 - \alpha)e^{-\alpha/(1 - \alpha)} > 1$, and for any non-negative sequence $t_{n}$ such that
$$n^{-1}\sum_{k = 1}^{n}\exp(-t_{n}\lambda_k)/\exp(-t_{n}\bar{\lambda}) \to 1,$$
we have $t_{n}\bar{\lambda} \to 0$.
\end{lemma}
\end{samepage}
\begin{lemma}\label{lemma: lip cont}
    Suppose $\sigma_{\varepsilon,n}^{2}$ is bounded away from infinity.
    Under Condition \ref{ass: bounded projection}, $|\mathbb{E}V_{n}(t_{1}) - \mathbb{E}V_{n}(t_{2})| \leq Cn\bar{\lambda}|t_{1} - t_{2}|$ for some universal constant $C$ and any $0\le t_{1}, t_{2} \leq \bar{\lambda}^{-1}$.
\end{lemma}
\begin{lemma}\label{lemma: concentration}
    Under Conditions \ref{ass: bounded projection} and \ref{ass: subgaussian}, with $\sigma_{\varepsilon,n}^{2}$ uniformly bounded above,
    $$
    \mathbb{P}(|V_{n}(t) - \mathbb{E}V_{n}(t)| \geq \sigma_{\varepsilon, n}^{2}\epsilon) \leq 4 \exp\left(-cn\sigma_{\varepsilon,n}^2\epsilon^{2}\right)
    $$
    for any $0\le t\le\bar\lambda^{-1}$, $0<\epsilon\le1$ and some universal constant $c$.
\end{lemma}
Let $\bm{A}_{t} = \bm{I} - \exp(-t\bm{H} / 2)$, $\mu_{1}(t) = n^{-1}\Tr\{\bm{A}_{t}\} = n^{-1} \sum_{k = 1}^{n}(1 - \exp(-t\lambda_k/2))$, $\mu_{2}(t) = n^{-1}\Tr\{\bm{A}_{t}^{2}\} = n^{-1} \sum_{k = 1}^{n}(1 - \exp(-t\lambda_k/2))^{2}$, and $b^{2}(t) = n^{-1}(f_{\star}(\bmX) - f_{0}(\bmX))^{\top}\exp(-t\bm{H})(f_{\star}(\bmX) - f_{0}(\bmX))$. 
Next, we provide the proof of Theorem \ref{thm: test error}.
\begin{proof}
    Note that $\mathcal{E}_{n}(t/2) = b^{2}(t) + \sigma_{\varepsilon, n}^{2}\mu_{2}(t) + \sigma_{\varepsilon, n}^{2}$ and 
    $\mathbb{E}V_{n}(t) = \exp(t\bar{\lambda})[b^{2}(t) + \sigma_{\varepsilon, n}^{2}\{1 - 2\mu_{1}(t) + \mu_{2}(t)\}]$. Thus, 
    $$
    \frac{\mathcal{E}_{n}(t/2) - \mathbb{E}V_{n}(t)}{\mathcal{E}_{n}(t/2)} = 1 - \exp(t\bar{\lambda}) + \frac{2\mu_{1}(t)\exp(t\bar{\lambda})\sigma_{\varepsilon, n}^{2}}{b^{2}(t) + \sigma_{\varepsilon, n}^{2}\mu_{2}(t) + \sigma_{\varepsilon, n}^{2}},
    $$
    which implies
    \begin{equation}\label{eq: relation E V}
        \begin{aligned}
        \left|\frac{\mathcal{E}_{n}(t/2) - \mathbb{E}V_{n}(t)}{\mathcal{E}_{n}(t/2)}\right| &\leq \exp(t\bar{\lambda}) - 1 + 2 \mu_{1}(t)\exp(t\bar{\lambda})\\
        &\leq \exp(t\bar{\lambda}) - 1 + 2 (1 - \exp(-t\bar{\lambda}/2))\exp(t\bar{\lambda})\\
        &=(3\exp(t\bar{\lambda}/2) + 1)(\exp(t\bar{\lambda}/2) - 1)\\
        &=: \nu(t),
    \end{aligned}
    \end{equation}
    where the second inequality is due to Jensen's inequality and the convexity of the exponential function.
    Let
    $$
    \hat t_n:=\hat t_n^{\rm REML},
    \qquad
    t_{\mathrm{opt}}:=\arg\min_{t\ge0}\mathcal{E}_n(t / 2),
    \qquad
    t_*:=\arg\min_{t\ge0}\mathbb{E}V_{n}(t).
    $$ 

Let $L_n=1+\log(1+\Gamma_n)$ and $t_c=\bar\lambda^{-1}\sqrt{L_n/r_n}$, with $\Gamma_n,r_n$ as in Condition~\ref{ass: learnable}. Then $t_c\bar\lambda\to0$ and
$$
\frac{b^2(t_c)}{\sigma_{\varepsilon,n}^2}
\le\frac{\sum_{\lambda_k\le r_n\bar\lambda}b_k^2}{n\sigma_{\varepsilon,n}^2}
 +\Gamma_ne^{-\sqrt{L_nr_n}}\to0,
$$
since $L_n/r_n\to0$ implies $L_n-\sqrt{L_nr_n}\to-\infty$.
Also $\mu_2(t_c)\le t_c\bar\lambda/2\to0$, so $\mathcal E_n(t_c/2)/\sigma_{\varepsilon,n}^2\to1$.

    Then, we have
    \begin{equation}\label{eq: relation E_opt V_opt}
          \mathcal{E}_{n}(t_{*}/2)\{1 - \nu(t_{*})\} \leq \mathbb{E}V_{n}(t_{*}) \leq \mathbb{E}V_{n}(t_c) \leq  \mathcal{E}_{n}(t_c/2)\{1 + \nu(t_c)\}.
    \end{equation}
    Thus, whenever $\nu(t_*)<1$ (as shown below),
    \begin{equation}\label{eq: bound pop risk ratio}
        \frac{\mathcal{E}_{n}(t_{*}/2)}{\mathcal{E}_{n}(t_{\rm opt}/2)} \leq \frac{\mathcal E_n(t_c/2)}{\mathcal E_n(t_{\rm opt}/2)}\frac{1 + \nu(t_c)}{1 - \nu(t_{*})}.
    \end{equation}
    Condition~\ref{ass: learnable} ensures that
$$
\frac{\mathcal{E}_n(t_{\mathrm{opt}} / 2)}{\sigma_{\varepsilon,n}^2}
=
\frac{\inf_{t\ge0}\mathcal{E}_n(t / 2)}{\sigma_{\varepsilon,n}^2}
\to 1.
$$
Using the above comparison \eqref{eq: relation E_opt V_opt} between $\mathcal{E}_n(t / 2)$ and $\mathbb{E}V_{n}(t)$, we have
$\limsup_{n\to \infty}\mathbb{E}V_{n}(t_*)/\sigma_{\varepsilon,n}^2\leq 1$. On the other hand, note that
$$
\mathbb{E}V_{n}(t) = \exp(t\bar{\lambda})[b^{2}(t) + \sigma_{\varepsilon, n}^{2}\{1 - 2\mu_{1}(t) + \mu_{2}(t)\}] \geq \sigma_{\varepsilon, n}^{2}\frac{n^{-1}\sum_{k = 1}^{n}\exp(-t\lambda_k)}{\exp(-t\bar{\lambda})} \geq \sigma_{\varepsilon, n}^{2}.
$$
Combining this with the relationship $\limsup_{n\to \infty}\mathbb{E}V_{n}(t_*)/\sigma_{\varepsilon,n}^2\leq 1$, we have $\lim_{n\to \infty}\mathbb{E}V_{n}(t_*)/\sigma_{\varepsilon,n}^2 = 1$ and
$$
\frac{n^{-1}\sum_{k = 1}^{n}\exp(-t_{*}\lambda_k)}{\exp(-t_{*}\bar{\lambda})} \to 1.
$$
This further implies $t_{*}\bar{\lambda} \to 0$ according to Lemma \ref{lemma: t_star control}.
Consequently,
\begin{equation}\label{eq: converge control}
    \nu(t_c)\to 0\ \text{and} \
    \nu(t_*)\to 0.
\end{equation}

Lemma~\ref{lemma: t_star control} additionally implies that
$$
\mathbb{E}V_{n}(\bar\lambda^{-1})
\ge (1+c)\sigma_{\varepsilon,n}^2
$$
for some $c>0$ eventually, whereas $\mathbb{E}V_{n}(t_*)=\sigma_{\varepsilon,n}^2(1+o(1))$. 
Lemma~\ref{lemma: concentration} therefore yields
$$
V_{n}(\bar\lambda^{-1})
=
\mathbb{E}V_{n}(\bar\lambda^{-1})
+
o_p(\sigma_{\varepsilon,n}^2),
\qquad
V_{n}(t_*)
=
\mathbb{E}V_{n}(t_*)
+
o_p(\sigma_{\varepsilon,n}^2),
$$
which implies $V_{n}(t_{*}) < V_{n}(\bar\lambda^{-1})$ with probability approaching one and hence $\mathbb{P}(\hat t_n\le\bar\lambda^{-1})\to1$ due to the convexity of $V_{n}(t)$.

On the interval $[0,\bar\lambda^{-1}]$, Lemma \ref{lemma: lip cont} shows that $\mathbb{E}V_{n}(t)$ is Lipschitz continuous with constant $Cn\bar{\lambda}$. 
Combining this property with the concentration result of Lemma~\ref{lemma: concentration} and the condition that $n^{a}\sigma_{\varepsilon, n}^{2} \to \infty$ for some $a\in (0, 1)$, standard coverage and union probability bound arguments give the following result. Indeed, on $u=t\bar\lambda\in[0,1]$, $|\partial_uV_n(u/\bar\lambda)|\le enV_n(0)=O_p(n)$. A grid of mesh $n^{-3}$ and Lemma~\ref{lemma: concentration} with $\epsilon_n=n^{-(1-a)/4}$ give a union bound tending to zero, since $n\sigma_{\varepsilon,n}^2\epsilon_n^2/\log n\to\infty$; interpolation contributes $O_p(n^{-2})=o_p(\sigma_{\varepsilon,n}^2)$. Thus,
\begin{equation}\label{eq: concentration}
    \sup_{0\le t\le \bar\lambda^{-1}}
|V_{n}(t)-\mathbb{E}V_{n}(t)|
=
o_p(\sigma_{\varepsilon,n}^2).
\end{equation}
Because $\hat t_n$ minimizes $V_n(t)$, we have
$V_{n}(\hat t_n)\le V_{n}(t_*)$.
Since $t_*$ minimizes $\mathbb{E}V_{n}(t)$, we have 
$$
\begin{aligned}
    \mathbb{E}V_{n}(t_{*}) 
    &\leq \mathbb{E}V_{n}(t)|_{t = \hat t_{n}}\\
    &\leq V_n(\hat t_n) + \left|V_n(\hat t_n) - \mathbb{E}V_{n}(t)|_{t = \hat t_{n}}\right|\\
    &\leq V_n(t_{*}) + \left|V_n(\hat t_n) - \mathbb{E}V_{n}(t)|_{t = \hat t_{n}}\right| \\
    &\leq \mathbb{E}V_{n}(t_{*}) +\left|V_n(t_{*}) - \mathbb{E}V_{n}(t_{*})\right|
    +\left|V_n(\hat t_n) - \mathbb{E}V_{n}(t)|_{t = \hat t_{n}}\right|\\
    &= \mathbb{E}V_{n}(t_{*}) + O_p\left(\sup_{0\le t\le \bar\lambda^{-1}}
|V_{n}(t)-\mathbb{E}V_{n}(t)|\right).
\end{aligned}
$$
Combining this with the concentration result \eqref{eq: concentration} gives
$$
\mathbb{E}V_{n}(t)|_{t = \hat t_{n}}-\mathbb{E}V_{n}(t_*)
=
o_p(\sigma_{\varepsilon,n}^2).
$$
Thus, $\mathbb{E}V_{n}(t)|_{t = \hat t_{n}} / \sigma_{\varepsilon,n}^2\to 1$ in probability and $\mathbb{E}V_{n}(t)|_{t = \hat t_{n}} / \mathbb{E}V_{n}(t_c) = 1 + o_{p}(1)$. Similar arguments as those in the proof of Lemma \ref{lemma: t_star control} using the pointwise bound \eqref{eq: zeta bound} show that $\widehat{t}_{n}\bar{\lambda} \to 0$ in probability.
Finally, applying the comparison inequality \eqref{eq: relation E V} between $\mathcal{E}_n(t / 2)$ and $\mathbb{E}V_{n}(t)$ yields
$$
\frac{\mathcal{E}_n(\hat t_n / 2)}{\mathcal{E}_n(t_{\mathrm{opt}} / 2)}
\le
\frac{\mathbb{E}V_{n}(t)|_{t = \hat t_{n}}}{\mathbb{E}V_{n}(t_c)}
\cdot
\frac{1+\nu(t_c)}{1-\nu(\hat t_n)}\cdot\frac{\mathcal E_n(t_c/2)}{\mathcal E_n(t_{\mathrm{opt}}/2)}
\leq
1+o_p(1),
$$
while the reverse inequality is immediate because $t_{\mathrm{opt}}$ minimizes $\mathcal{E}_n(t / 2)$. Therefore under condition \ref{ass: learnable}
$$
\sigma_{\varepsilon, n}^{-2}\mathcal{E}_n(\hat t_n / 2) =
\frac{\mathcal{E}_n(\hat t_n / 2)}{\mathcal{E}_n(t_{\rm opt} / 2)} \frac{\mathcal{E}_n(t_{\rm opt} / 2)}{\sigma_{\varepsilon, n}^2}
\to 1
$$
in probability. This implies that $\sigma_{\varepsilon, n}^{-2}b^2(\widehat{t}_{n}) \to 0$ and $\mu_{2}(\hat{t}_{n}) \to 0$ in probability. Note that, $\hat{t}_{n} \geq 0$, $b^2(t)$ is decreasing in $t$, and 
$$
\begin{aligned}
    \mu_{2}(2\widehat{t}_{n}) 
    &= n^{-1} \sum_{k = 1}^{n}(1 - \exp(-\widehat t_n\lambda_k))^{2}\\
    &= n^{-1} \sum_{k = 1}^{n}(1 - \exp(-\widehat t_n\lambda_k/2))^{2}(1 + \exp(-\widehat t_n\lambda_k/2))^{2}\\
    &\leq 4n^{-1} \sum_{k = 1}^{n}(1 - \exp(-\widehat t_n\lambda_k/2))^{2}\\
    & =  4\mu_{2}(\widehat{t}_{n}).
\end{aligned}
$$
We have $\sigma_{\varepsilon, n}^{-2}b^2(2\widehat{t}_{n}) = o_p(1)$ and $\mathcal{E}_n(\hat t_n) = b^2(2\widehat{t}_{n}) + \sigma_{\varepsilon, n}^{2}\mu_{2}(2\widehat{t}_{n}) + \sigma_{\varepsilon, n}^{2} = \sigma_{\varepsilon, n}^{2} + o_p(\sigma_{\varepsilon, n}^{2})$ and hence
$$
\frac{\mathcal{E}_n(\hat t_n)}{\inf_{t\ge0}\mathcal{E}_n(t)} = \sigma_{\varepsilon, n}^{-2}\mathcal{E}_n(\hat t_n)\{1+o(1)\}
\to 1
$$
in probability.
This completes the proof.
\end{proof}

\section{Proof of Lemma \ref{lemma: t_star control}}
\begin{proof}
Set $a_{n}:=t_n\bar{\lambda}$ and $x_{k,n}:=\lambda_{k}/\bar{\lambda}$, so that
$x_{k,n}\ge 0$ and $n^{-1}\sum_{k=1}^{n} x_{k,n}=1$. Let 
$$\zeta_{n} = n^{-1}\sum_{k = 1}^{n}\exp(-t_n\lambda_k)/\exp(-t_n\bar{\lambda}).
$$
Then
$$
\zeta_{n}
=
e^{a_{n}}\frac1n\sum_{k=1}^{n} e^{-a_{n} x_{k,n}}.
$$
The condition becomes
$$
\frac1n\sum_{k=1}^{n} \mathbf 1\{x_{k,n}\le \delta_n\}\ge \alpha.
$$

For all large $n$, choose a subset
$B_n\subset\{1,\dots,n\}$ with $|B_n|=m_n$, where $m_n/n=:\alpha_n\to\alpha$,
such that $x_{k,n}\le \delta_n$ for all $k\in B_n$.
Let
$$
u_n:=\frac1{m_n}\sum_{k\in B_n}x_{k,n},\qquad
v_n:=\frac1{n-m_n}\sum_{k\notin B_n}x_{k,n}.
$$
Then $0\le u_n\le \delta_n$ and
$$
\alpha_n u_n+(1-\alpha_n)v_n=1,
\qquad
v_n=\frac{1-\alpha_n u_n}{1-\alpha_n}.
$$

By Jensen's inequality applied on $B_n$ and $B_n^c$,
$$
\frac1n\sum_{k=1}^{n} e^{-a_{n} x_{k,n}}
\ge
\alpha_n e^{-a_{n} u_n}+(1-\alpha_n)e^{-a_{n} v_n}.
$$
Hence
$$
\zeta_{n}
\ge
\alpha_n e^{a_{n}(1-u_n)}
+
(1-\alpha_n)e^{-a_{n}(v_n-1)}.
$$

Using
$$
v_n-1=\frac{\alpha_n(1-u_n)}{1-\alpha_n},
$$
we obtain
\begin{equation}\label{eq: zeta bound}
    \zeta_{n}
\ge
\alpha_n e^{(1-u_n)a_{n}}
+
(1-\alpha_n)e^{-\frac{\alpha_n}{1-\alpha_n}(1-u_n)a_{n}}
\ge
\alpha_n e^{(1-\delta_n)a_{n}}
+
(1-\alpha_n)e^{-\frac{\alpha_n}{1-\alpha_n}(1-\delta_n)a_{n}}.
\end{equation}
Taking $t_n = 1 / \bar{\lambda}$, we have $\liminf_{n\to\infty}n^{-1}\sum_{k = 1}^{n}\exp(-\bar{\lambda}^{-1}\lambda_k+1) = \liminf_{n\to\infty}\zeta_{n} \geq \alpha e+(1 - \alpha)e^{-\alpha/(1 - \alpha)} > 1$ by the strict convexity of the exponential function, which proves the first claim of the lemma.

Next, we prove the second claim of the lemma.
Assume for contradiction that $a_{n}\not\to 0$. Then along some subsequence,
still denoted by $n$, we have $a_{n}\to s\in(0,\infty]$.
Since $\alpha_n\to\alpha\in(0,1)$ and $\delta_n\to 0$, the right-hand side of
\eqref{eq: zeta bound} converges to
$$
\alpha e^s+(1-\alpha)e^{-\frac{\alpha}{1-\alpha}s}
$$
if $s<\infty$, and to $+\infty$ if $s=\infty$.
For every $s>0$,
$$
\alpha e^s+(1-\alpha)e^{-\frac{\alpha}{1-\alpha}s}>1
$$
by the strict convexity of the exponential function. Therefore
$$
\limsup_{n\to\infty}\zeta_{n}>1,
$$
which contradicts $\zeta_{n}\to 1$.
Thus $a_{n}=t_n\bar{\lambda}\to 0$.
\end{proof}

\section{Proof of Lemma \ref{lemma: lip cont}}
\begin{proof}
Note that
$$
\mathbb{E}V_{n}(t)=e^{t\bar\lambda}\Big\{b^2(t)+\sigma_{\varepsilon,n}^2 s(t)\Big\},
$$
where
$s(t):=n^{-1}\sum_{k=1}^{n} e^{-t\lambda_k}$,
$
b^2(t)=\frac1n (f_{\star}(\bmX)-f_0(\bmX))^\top \exp(-t\bm{H})(f_{\star}(\bmX)-f_0(\bmX))$.
Recall that
$b_k=\bmv_k^\top (f_{\star}(\bmX)-f_0(\bmX))
$.
Then
$
b^2(t)=\frac1n\sum_{k=1}^{n} b_k^2 e^{-t\lambda_k}
$.

Differentiating $\mathbb{E}V_{n}(t)$ gives
$$
\frac{d}{dt}\mathbb{E}V_{n}(t)
=
e^{t\bar\lambda}\Big[
\bar\lambda\{b^2(t)+\sigma_{\varepsilon,n}^2 s(t)\}
-\frac1n\sum_{k=1}^{n} \lambda_k b_k^2 e^{-t\lambda_k}
-\sigma_{\varepsilon,n}^2 \frac1n\sum_{k=1}^{n} \lambda_k e^{-t\lambda_k}
\Big].
$$

For $t\le \bar\lambda^{-1}$ we have $e^{t\bar\lambda}\le e$. 
Moreover,
$$
0\le s(t)\le 1,
\qquad
\left|\frac1n\sum_{k=1}^{n} \lambda_k e^{-t\lambda_k}\right|\le \bar{\lambda}.
$$

Under Condition \ref{ass: bounded projection}, the projection coefficients satisfy
$
n^{-1}\sum_{k=1}^{n} b_k^2 \le C
$
for some constant $C>0$. Consequently, 
$$
b^2(t)\le C,
\qquad
\left|\frac{d}{dt}b^2(t)\right|
=
\left|\frac1n\sum_{k=1}^{n} \lambda_k b_k^2 e^{-t\lambda_k}\right|
\le C\lambda_1 \le Cn\bar{\lambda}
$$
according to Condition \ref{ass: bounded projection}.

Combining these bounds yields
$$
\left|\frac{d}{dt}\mathbb{E}V_{n}(t)\right|\le Cn\bar{\lambda}
$$
for all $t\le \bar\lambda^{-1}$. The result follows immediately from the mean value theorem.
\end{proof}

\section{Proof of Lemma \ref{lemma: concentration}}
\begin{proof}
Let
$$
\bmr=f_{\star}(\bmX)-f_0(\bmX), 
\qquad
\bm{B}_{t} := \frac1n e^{t\bar\lambda} \exp(-t\bm{H}).
$$
Then
$$
V_{n}(t)=(\bmr+\bm{\varepsilon})^\top \bm{B}_{t} (\bmr+\bm{\varepsilon}),
$$
and therefore
$$
V_{n}(t)-\mathbb{E}V_{n}(t)
=
2\bmr^\top \bm{B}_{t}\bm{\varepsilon}
+
\Big(\bm{\varepsilon}^\top \bm{B}_{t}\bm{\varepsilon}
      -\mathbb{E}[\bm{\varepsilon}^\top \bm{B}_{t}\bm{\varepsilon}]\Big).
$$

For $t\le \bar\lambda^{-1}$, we have
$$
\|\bm{B}_{t}\|_{\mathrm{op}}
=
\frac1n \max_{1\le k\le n} e^{t(\bar\lambda-\lambda_k)}
\le \frac{e}{n},
$$
and
$$
\|\bm{B}_{t}\|_F^2
=
\frac1{n^2}\sum_{k=1}^{n} e^{2t(\bar\lambda-\lambda_k)}
\le \frac{e^2}{n},
$$
where $\|\cdot\|_{op}$ and $\|\cdot\|_{F}$ are the operator and Frobenius norm, respectively.

Condition~\ref{ass: subgaussian} implies that $\sigma_{\varepsilon,n}^{-1}\varepsilon_i$ are uniformly sub-Gaussian. Applying the Hanson--Wright inequality of \cite{HWinequality} gives
$$
\mathbb{P}\!\left(
\left|
\bm{\varepsilon}^\top \bm{B}_{t}\bm{\varepsilon}
-
\mathbb{E}(\bm{\varepsilon}^\top \bm{B}_{t}\bm{\varepsilon})
\right|
\ge \frac12\sigma_{\varepsilon,n}^2\epsilon
\right)
\le
2\exp(-c_1 n\epsilon^2)
$$
for some universal constant $c_1>0$ and $0<\epsilon\le1$.

For the linear term, $\bmr^\top \bm{B}_{t}\bm{\varepsilon}$ is sub-Gaussian with variance proxy bounded by
$$
\sigma_{\varepsilon,n}^2 \|\bm{B}_{t} \bmr\|_2^2
\le
\sigma_{\varepsilon,n}^2
\|\bm{B}_{t}\|_{\mathrm{op}}^2
\|\bmr\|_2^2
\le
C\frac{\sigma_{\varepsilon,n}^2}{n},
$$
where Condition \ref{ass: bounded projection} ensures $\|\bmr\|_2^2/n$ is bounded. Consequently,
$$
\mathbb{P}\!\left(
|2\bmr^\top \bm{B}_{t}\bm{\varepsilon}|
\ge \frac12\sigma_{\varepsilon,n}^2\epsilon
\right)
\le
2\exp(-c_2 n\sigma_{\varepsilon,n}^2\epsilon^2)
$$
for some $c_2>0$.

Combining the two bounds yields
$$
\mathbb{P}\!\left(
|V_{n}(t)-\mathbb{E}V_{n}(t)|
\ge
\sigma_{\varepsilon,n}^2\epsilon
\right)
\le
4\exp(-c n\sigma_{\varepsilon,n}^2\epsilon^2)
$$
for some universal constant $c>0$ because $\sigma_{\varepsilon,n}^2$ is bounded, for $0<\epsilon\le1$.
\end{proof}

\section{Proof of Corollary \ref{cor:out_sample_optimality}}
We first introduce some conditions and notations required in the proof of Corollary \ref{cor:out_sample_optimality}.
For the random-design analysis, assume $\bmx_i\overset{i.i.d.}{\sim}\mathbb{P}_X$. For any measurable function $f$, define
$
\mathbb{P}_nf=n^{-1}\sum_{i=1}^n f(\bmx_i)$ and
$\mathbb{P}_Xf=\mathbb{E}\{f(\bm x)\}$ with $\bm x\sim \mathbb{P}_X$.

Let $T$ be the population fixed operator under $\mathbb{P}_X$,
\begin{equation*}
(Tf)(\bm x)
=
\int h(\bm x,\bm x')f(\bm x')\,d\mathbb{P}_X(\bm x'),
\qquad f\in L_2(\mathbb{P}_X).
\end{equation*}

Let $(\mu_k,\phi_k)_{k\ge1}$ be the eigenpairs of $T$, with
$\infty>\mu_1\ge\mu_2\ge\cdots>0$,
$\mathbb{P}_X\phi_k^2=1$, and $\{\phi_k\}_{k\geq 1}$ form orthogonal bases for $L_2(\mathbb{P}_X)$.  Let $\cH$ be the reproducing kernel Hilbert space (RKHS) associated with $h(\cdot,\cdot)$.
Write
\begin{equation*}
f_\star-f_0=\sum_{k\ge1}\beta_k\phi_k.
\end{equation*}

\begin{condition}[Fixed-operator spectral properties]\label{ass: NTK spectral}
The function $f_\star-f_0$ is fixed-operator-smooth in the sense that $\sum_{k\ge1}\beta_k^2/\mu_k<\infty$. In addition, there exists a deterministic integer sequence $\{K_n\}_{n\ge1}$, with $1\le K_n\le n$, such that: 
  
\begin{itemize}
    \item[(i)] 
    \begin{equation}\label{eq: spectral tail}
\mu_{\lfloor\sqrt{n}\rfloor}=o\left(n^{-2}\right),~\sum_{k>\lfloor\sqrt{n}\rfloor}\mu_k=o\left(n^{-1}\right),~
\sum_{k>K_n}\frac{\beta_k^2}{\mu_k}=O\left(n^{-1/2}\right)~ \text{and}~ \mu_{K_n}\gg\frac{\log n}{\sqrt n} 
\end{equation}
where $\lfloor\cdot\rfloor$ denotes rounded down;
    \item[(ii)] The signal-relevant eigenvalues are separated. Let $\delta_n=\min_{1\le k\le K_n} \min\{\mu_{k-1}-\mu_k,\mu_k-\mu_{k+1}\}$ with the convention $\mu_0=\infty$, then $\delta_n\geq c_0n^{-1/4}$ for some positive constant $c_0$.
    \item[(iii)] The population eigenfunctions are uniformly bounded, i.e., $\max_{k}\|\phi_k\|_\infty=O(1)$.
\end{itemize}

\end{condition}

\begin{lemma}[Empirical fixed operator spectral approximation] \label{lem: emp NTK spectral}
Suppose $f_\star-f_0\in L_2(\mathbb{P}_X)$ and $h(\cdot,\cdot)$ is a bounded Mercer kernel. Under Condition \ref{ass: NTK spectral}, after choosing the signs of $\bm v_k$ consistently with $\phi_k$, we have 
\begin{equation}\label{eq: spectral property}
\begin{aligned}
    &\max_{k > \lfloor\sqrt{n}\rfloor}|\lambda_k/n - \mu_k| = o_p(n^{-1}),\quad
\max_{1\le k\le K_n}
\left|
\frac{\lambda_k}{n\mu_k}-1
\right|=o_p(1), \\ &\max_{1\le k\le K_n}
\left|
\frac{b_k}{\sqrt n}
-\beta_k
\right|=o_p(1), \quad  \bar\lambda -  \operatorname{Tr}(T) = O_p(n^{-1/2}).
\end{aligned}
\end{equation}
Moreover, we have
\begin{equation*}
\frac1n\sum_{k=1}^{n} b_k^2\exp(-\lambda_k/\sqrt n)
=
O_P(n^{-1}).
\end{equation*}
\end{lemma}
\begin{proof}
The first result in this lemma follows from Condition \ref{ass: NTK spectral} (i) and Theorem 3 of \citep{braun2006accurate} by taking $r$ therein as $\lfloor\sqrt{n}\rfloor$.

We apply spectral perturbation arguments on the common Hilbert space $\cH$.
Write $h_{\bmx}=h(\cdot,\bmx)$ and let
$S_n:\cH\to\mathbb R^n$ be the evaluation operator
$S_nf=\{f(\bmx_1),\ldots,f(\bmx_n)\}^{\top}$.
Define the population and empirical RKHS covariance operators by
\begin{equation*}
\mathsf C=\mathbb E\{h_{\bmx}\otimes h_{\bmx}\},
\qquad
\mathsf C_n=\frac1n\sum_{i=1}^n h_{\bmx_i}\otimes h_{\bmx_i}
=\frac1n S_n^*S_n,
\end{equation*}
where $(u\otimes u)v=\langle u,v\rangle_{\cH}u$ and $S_n^*$ is the adjoint of $S_n$.
Both operators are positive and self-adjoint on $\cH$.
The nonzero eigenvalues of $\mathsf C$ are the eigenvalues $\mu_k$ of $T$,
and the corresponding $\cH$-normalized eigenfunctions are
$g_k=\sqrt{\mu_k}\phi_k$.
Since $\bm H=S_nS_n^*$, the nonzero eigenvalues of $\mathsf C_n$ are
$\widehat\mu_k=\lambda_k/n$.
For each $\lambda_k>0$, the corresponding normalized eigenfunction and its rescaling are
\begin{equation*}
\widehat g_k=\lambda_k^{-1/2}S_n^*\bm v_k,
\qquad
\widehat\phi_k=\frac{\widehat g_k}{\sqrt{\widehat\mu_k}}
=\frac{\sqrt n}{\lambda_k}h(\cdot,\bm X)\bm v_k.
\end{equation*}
In particular, $\|g_k\|_{\cH}=\|\widehat g_k\|_{\cH}=1$ and
$\widehat\phi_k(\bmx_\ell)=\sqrt n\,v_{k\ell}$, so
$\bm v_k=n^{-1/2}\widehat\phi_k(\bm X)$.

Let $\kappa^2=\sup_{\bmx}h(\bmx,\bmx)<\infty$ and
$\eta_n=\|\mathsf C_n-\mathsf C\|_{\mathrm{HS}}$.
Independence of the covariates gives
\begin{equation*}
\mathbb E\eta_n^2
=\frac1n\mathbb E\|h_{\bmx}\otimes h_{\bmx}-\mathsf C\|_{\mathrm{HS}}^2
\le\frac{\kappa^4}{n}.
\end{equation*}
Thus $\|\mathsf C_n-\mathsf C\|_{\mathrm{op}}\le\eta_n=O_p(n^{-1/2})$.
By Weyl's inequality,
\begin{equation*}
\max_{1\le k\le K_n}|\widehat\mu_k-\mu_k|
\le\|\mathsf C_n-\mathsf C\|_{\mathrm{op}}\le\eta_n.
\end{equation*}
Condition~\ref{ass: NTK spectral} (i) implies
$\eta_n/\mu_{K_n}=o_p(1)$, and hence
\begin{equation*}
\max_{1\le k\le K_n}
\left|\frac{\lambda_k}{n\mu_k}-1\right|=o_p(1).
\end{equation*}

To compare eigenfunctions, apply the spectral perturbation bound in
Theorem~2 of \cite{zwald2005convergence} to the normalized eigenfunctions
$g_k$ and $\widehat g_k$ of $\mathsf C$ and $\mathsf C_n$.
Since $\eta_n/\delta_n=o_p(1)$, signs can be chosen consistently so that,
with probability tending to one,
\begin{equation*}
\max_{1\le k\le K_n}\|\widehat g_k-g_k\|_{\cH}
\le\frac{C\eta_n}{\delta_n},
\qquad
\widehat\mu_k\ge\mu_k/2\quad(1\le k\le K_n).
\end{equation*}
The rescaling from these unit eigenfunctions to $\phi_k$ and $\widehat\phi_k$ yields
\begin{equation*}
\begin{aligned}
\|\widehat\phi_k-\phi_k\|_{\cH}
&\le\frac{\|\widehat g_k-g_k\|_{\cH}}{\sqrt{\widehat\mu_k}}
+\left|\widehat\mu_k^{-1/2}-\mu_k^{-1/2}\right|\\
&\le\frac{C\eta_n}{\delta_n\sqrt{\mu_{K_n}}}
+\frac{C\eta_n}{\mu_{K_n}^{3/2}},\qquad 1\le k\le K_n.
\end{aligned}
\end{equation*}
Condition~\ref{ass: NTK spectral} (ii) gives
$\mu_{K_n}\ge\mu_{K_n}-\mu_{K_n+1}\ge\delta_n\ge c_0n^{-1/4}$.
Consequently,
\begin{equation*}
\max_{1\le k\le K_n}\|\widehat\phi_k-\phi_k\|_{\cH}
=O_p(n^{-1/8})=o_p(1).
\end{equation*}
By the reproducing property,
$\|\widehat\phi_k-\phi_k\|_\infty\le\kappa\|\widehat\phi_k-\phi_k\|_{\cH}$.
This also gives, uniformly over $1\le k\le K_n$,
$\|\widehat\phi_k-\phi_k\|_{L_2(\mathbb P_X)}=o_p(1)$ and
$\|\widehat\phi_k-\phi_k\|_{L_2(\mathbb P_n)}=o_p(1)$.

 The leading eigenfunctions are uniformly bounded, and
$
K_n\mu_{K_n}\le \sum_{k=1}^{K_n}\mu_k\le \operatorname{Tr}(T)<\infty$.
Together with $\mu_{K_n}\gg \log n/\sqrt n$, this gives
$
K_n=o\left(\sqrt n/\log n\right)$
and hence
$K_n\log K_n/n=o(1)$.

Now consider the empirical projection of $f_\star-f_0$. Because $\bm v_k=n^{-1/2}\widehat\phi_k(\bm X)$, we have
\begin{equation*}
\frac{b_k}{\sqrt n}
=
\frac{\bm v_k^\top\{f_\star(\bm X)-f_0(\bm X)\}}{\sqrt n}
=
\mathbb{P}_n\{\widehat\phi_k(f_\star-f_0)\}.
\end{equation*}

Thus,
\begin{equation*}
\mathbb{P}_n\{\widehat\phi_k(f_\star-f_0)\}-\beta_k
=
\mathbb{P}_n\{(\widehat\phi_k-\phi_k)(f_\star-f_0)\}
+
(\mathbb{P}_n-\mathbb{P}_X)\{\phi_k(f_\star-f_0)\}.
\end{equation*}

The first term is $o_p(1)$ uniformly over $k\le K_n$ by Cauchy--Schwarz, because
\begin{equation*}
\|\widehat\phi_k-\phi_k\|_{L_2(\mathbb{P}_n)}=o_p(1)
\end{equation*}
uniformly over $k\le K_n$, and $f_\star-f_0\in L_2(\mathbb{P}_X)$. The second term is $o_p(1)$ uniformly over $k\le K_n$ by a uniform law of large numbers for the finite class
\begin{equation*}
\{\phi_k(f_\star-f_0):1\le k\le K_n\},
\end{equation*}
using boundedness of the leading $\phi_k$'s and $K_n\log K_n/n=o(1)$. Therefore,
\begin{equation*}
\max_{1\le k\le K_n}
\left|
\frac{b_k}{\sqrt n}
-\beta_k
\right|=o_p(1).
\end{equation*}

For the average eigenvalue,
\begin{equation*}
\bar\lambda
=
\frac1n\operatorname{Tr}(\bm H)
=
\frac1n\sum_{i=1}^n h(\bmx_i,\bmx_i).
\end{equation*}
Note that $\operatorname{Tr}(T) = E\{h(\bm x,\bm x)\}$
By the central limit theorem,
$\bar\lambda -
E\{h(\bm x,\bm x)\}
=
O_p(n^{-1/2})$.

Next, we prove the final result of this lemma. Let $r=f_\star-f_0$. Under Condition \ref{ass: NTK spectral}, 
$r\in\mathcal H$ and
$$
\|r\|_{\mathcal H}^2
=
\sum_{k\ge1}\frac{\beta_k^2}{\mu_k} < \infty.
$$
Recall that $\bm H=S_nS_n^*$ for the evaluation operator introduced above.
For each $\lambda_k>0$, the normalized empirical eigenfunction satisfies
$$
\hat{g}_k=\lambda_k^{-1/2}S_n^* \bmv_k 
=
\sqrt{\frac{\lambda_k}{n}}\widehat\phi_k
=
\lambda_k^{-1/2}h(\cdot,\bmX)\bmv_k.$$
Because $\bm{H} \bmv_k=\lambda_k \bmv_k$, we have
$$
\langle \hat{g}_k,\hat{g}_j\rangle_{\mathcal H}
=
\frac{\bmv_k^\top \bm{H} \bmv_j}{\sqrt{\lambda_k\lambda_j}}
=
\mathbf 1\{k=j\}.
$$
Thus $\{\hat{g}_k\}_{k=1}^{n}$ is an orthonormal system in
$\mathcal H$. Moreover, with
$b_k=\bmv_k^\top\{f_\star(\bmX)-f_0(\bmX)\}$, we have
$$
b_k
=
\bmv_k^\top S_n r
=
\langle S_n^* \bmv_k,r\rangle_{\mathcal H}
=
\sqrt{\lambda_k}\langle \hat{g}_k,r\rangle_{\mathcal H}.
$$
Therefore, by Bessel's inequality,
$$
\sum_{k = 1}^{n}\frac{b_k^2}{\lambda_k}
=
\sum_{k = 1}^{n}\langle \hat{g}_k,r\rangle_{\mathcal H}^2
\le
\|r\|_{\mathcal H}^2
=
\sum_{k\ge1}\frac{\beta_k^2}{\mu_k}.
$$
Thus, using
$$
\sup_{x\ge0}x\exp(-x/\sqrt n)=\frac{\sqrt n}{e},
$$
we obtain
$$
\begin{aligned}
\frac1n\sum_{k=1}^{n} b_k^2\exp(-\lambda_k/\sqrt n)
&=
\frac1n\sum_{k = 1}^{n}
\frac{b_k^2}{\lambda_k}\lambda_k\exp(-\lambda_k/\sqrt n)\\
&\leq
\frac1n\sum_{k = 1}^{K_n}
\frac{b_k^2}{\lambda_k}\lambda_k\exp(-\lambda_k/\sqrt n) + 
\frac1n\sum_{k = K_n + 1}^{n}
\frac{b_k^2}{\lambda_k}\lambda_k\exp(-\lambda_k/\sqrt n)\\
&\leq
\frac1n\sum_{k = 1}^{K_n}
\frac{b_k^2}{\lambda_k}\lambda_k\exp(-n\log n/ n) + 
\frac1n\sum_{k = K_n + 1}^{n}
\frac{b_k^2}{\lambda_k}\lambda_k\exp(-\lambda_k/\sqrt n)\\
&\le \frac1n\sum_{k = 1}^{n}
b_k^2 \times \frac{1}{n}
+
\frac1n
\left\{\sup_{x\ge0}x\exp(-x/\sqrt n)\right\}
\sum_{k = K_n + 1}^{n}\frac{b_k^2}{\lambda_k}\\
&\le \frac1n\sum_{k = 1}^{n}
b_k^2 \times \frac{1}{n}
+
\frac1n
\left\{\sup_{x\ge0}x\exp(-x/\sqrt n)\right\}
\left(\sum_{k\ge1}\frac{\beta_k^2}{\mu_k} - \sum_{k = 1}^{K_n}\frac{b_k^2}{\lambda_k}\right)\\
&\le O_p\left(\frac{1}{n}\right) + 
\frac{1}{e\sqrt n}
\left(\sum_{k\ge1}\frac{\beta_k^2}{\mu_k} - \sum_{k = 1}^{K_n}\frac{b_k^2}{\lambda_k}\right)
\end{aligned}
$$
with probability approaching one
according to \eqref{eq: spectral tail}.
Then, the proof is completed if we prove 
$$
\sum_{k\ge1}\frac{\beta_k^2}{\mu_k}
-
\sum_{k=1}^{K_n}\frac{b_k^2}{\lambda_k} = O_p\left(n^{-1/2}\right).
$$
Note that $g_k=\sqrt{\mu_k}\phi_k$
forms an orthonormal system in $\mathcal H$.
Let $\Pi_{K_n}$ be the orthogonal projection onto
$
\operatorname{span}(g_1,\ldots,g_{K_n})$,
and let $\widehat\Pi_{K_n}$ be the orthogonal projection onto
$
\operatorname{span}(\hat g_1,\ldots,\hat g_{K_n})$.
Then,
$$
\sum_{k\ge1}\frac{\beta_k^2}{\mu_k}
-
\sum_{k=1}^{K_n}\frac{b_k^2}{\lambda_k}
=
\|r\|_{\mathcal H}^2
-
\|\widehat\Pi_{K_n}r\|_{\mathcal H}^2
=
\|(I-\widehat\Pi_{K_n})r\|_{\mathcal H}^2 .
$$
Now, decompose
$(I-\widehat\Pi_{K_n})r
=
(I-\Pi_{K_n})r
+
(\Pi_{K_n}-\widehat\Pi_{K_n})r$.
Thus,
$$
\|(I-\widehat\Pi_{K_n})r\|_{\mathcal H}^2
\le
2\|(I-\Pi_{K_n})r\|_{\mathcal H}^2
+
2\|(\Pi_{K_n}-\widehat\Pi_{K_n})r\|_{\mathcal H}^2.
$$
The first term
$
\|(I-\Pi_{K_n})r\|_{\mathcal H}^2
=
\sum_{k>K_n}\beta_k^2/\mu_k = O(1/\sqrt{n})$ according to Condition \ref{ass: NTK spectral} (i).

For the second term, apply the spectral perturbation bound to the RKHS covariance
operators $\mathsf C$ and $\mathsf C_n$ and their leading eigenspaces
$\operatorname{span}(g_1,\ldots,g_{K_n})$ and
$\operatorname{span}(\widehat g_1,\ldots,\widehat g_{K_n})$.
Their population spectral gap is $\mu_{K_n}-\mu_{K_n+1}\ge\delta_n$.
Since $\eta_n/\delta_n=o_p(1)$, Theorem~3 of \cite{zwald2005convergence} gives,
with probability tending to one,
$$
\|\widehat\Pi_{K_n}-\Pi_{K_n}\|_{op}
\le
\frac{C\|\mathsf C_n-\mathsf C\|_{\mathrm{HS}}}{\mu_{K_n}-\mu_{K_n+1}}
\le\frac{C\eta_n}{\delta_n}=O_p(n^{-1/4})
$$
by Condition~\ref{ass: NTK spectral} (ii). Hence
$$
\|(\Pi_{K_n}-\widehat\Pi_{K_n})r\|_{\mathcal H} \leq \|\widehat\Pi_{K_n}-\Pi_{K_n}\|_{op}\|r\|_{\mathcal H} = O_p(n^{-1/4})
.
$$

Thus,
$$
\sum_{j\ge1}\frac{\beta_j^2}{\mu_j}
-
\sum_{k=1}^{K_n}\frac{b_k^2}{\lambda_k}
=
\|r\|_{\mathcal H}^2
-
\|\widehat\Pi_{K_n}r\|_{\mathcal H}^2
=
\|(I-\widehat\Pi_{K_n})r\|_{\mathcal H}^2 = O_p\left(n^{-1/2}\right).
$$
This completes the proof.
\end{proof}

\begin{lemma}[Order of the REML-guided stopping time] \label{lem: order of stop time}
Under the conditions of Lemma \ref{lemma: concentration} and Lemma \ref{lem: emp NTK spectral}, and if $\sqrt n\,\sigma_{\varepsilon,n}^{2}\to\infty$, we have
\begin{equation*}
\widehat t_n^{\mathrm{REML}}
=
O_p\left(
n^{-1/2}
\right).
\end{equation*}
\end{lemma}

\begin{proof}
Recall that $\widehat{t}_{n}^{\rm REML}$ is the minimizer of the convex function
$$
\begin{aligned}
    V_n(t) 
    & = n^{-1}\exp\{n^{-1}Q(t)\}\\
    & = n^{-1}(\bmy - f_{0}(\bmX))^\top \exp(-t \bm{H}) (\bmy - f_{0}(\bmX))\exp(t\bar{\lambda}),
\end{aligned}
$$
and 
$\mathbb{E}\{V_{n}(t)\mid\bmX\} = \exp(t\bar{\lambda})[b^{2}(t) + \sigma_{\varepsilon, n}^{2}\{1 - 2\mu_{1}(t) + \mu_{2}(t)\}] = n^{-1}\exp(t\bar{\lambda})[\sum_{k=1}^{n} b_k^2\exp(-t\lambda_k) + \sigma_{\varepsilon, n}^{2}\sum_{k=1}^{n}\exp(-t\lambda_k)]$. By Condition \ref{ass: NTK spectral} and Lemma \ref{lem: emp NTK spectral}, $\bar{\lambda}$ converges to a fixed constant and $\lambda_{\lfloor\sqrt{n}\rfloor+1} < \bar{\lambda} / 2$ with probability approaching one. For any sequence $\tau_{n}$ such that $\sqrt{n}\tau_n \to \infty$ and $\tau_n\to0$,
we have 
$$
\begin{aligned}
    &\mathbb{E}\{V_{n}(\tau_{n})\mid\bmX\} - \mathbb{E}\{V_{n}(1/\sqrt{n})\mid\bmX\}\\
    &\geq \exp(\tau_n\bar{\lambda})\left[\sigma_{\varepsilon, n}^{2}n^{-1}\sum_{k=\lfloor\sqrt{n}\rfloor+1}^{n}\exp(-\tau_n\lambda_k)\right] \\
    & \quad{} - \exp(\bar{\lambda}/\sqrt{n})\left[n^{-1}\sum_{k=1}^{n} b_k^2\exp(-\lambda_k/\sqrt{n}) +  \sigma_{\varepsilon, n}^{2}n^{-1/2}+ \sigma_{\varepsilon, n}^{2}n^{-1}\sum_{k=\lfloor\sqrt{n}\rfloor+1}^{n}\exp(-\lambda_k/\sqrt{n})\right]\\
    &\geq \sigma_{\varepsilon, n}^{2}n^{-1}\sum_{k=\lfloor\sqrt{n}\rfloor+1}^{n}\left[\exp\{\tau_n(\bar{\lambda} - \lambda_k)\} - \exp\{(\bar{\lambda} - \lambda_k)/\sqrt{n}\}\right] + O_p(n^{-1} + \sigma_{\varepsilon, n}^{2}n^{-1/2})\\
    &\geq \sigma_{\varepsilon, n}^{2}n^{-1}(n - \sqrt{n})\frac{\bar\lambda}{2}\tau_n\{1 - 1/(\sqrt{n}\tau_n)\} + O_p(n^{-1} + \sigma_{\varepsilon, n}^{2}n^{-1/2})\\
    & = \sigma_{\varepsilon, n}^{2}\tau_n\frac{\operatorname{Tr}(T)}{2} + o_{p}(\sigma_{\varepsilon, n}^{2}\tau_n)
\end{aligned}
$$
according to Lemma \ref{lem: emp NTK spectral} and $n^{1/2}\sigma_{\varepsilon,n}^{2} \to \infty$.

To control the stochastic increment, set $s_n=n^{-1/2}$, $\bmr=f_\star(\bmX)-f_0(\bmX)$ and
$$
\bm D_n=n^{-1}\{e^{\tau_n\bar\lambda}e^{-\tau_n\bm H}
                  -e^{s_n\bar\lambda}e^{-s_n\bm H}\}.
$$
Splitting the spectrum at $\lfloor\sqrt n\rfloor$ and using
$\max_{k>\lfloor\sqrt n\rfloor}\lambda_k=o_p(1)$ gives
$\|\bm D_n\|_F^2=O_p(n^{-3/2}+\tau_n^2/n)$.
Its eigenvalue in direction $\bmv_k$ is bounded in absolute value by $Cn^{-1}e^{-s_n\lambda_k}$ with probability tending to one, so
$\|\bm D_n\bmr\|^2\le Cn^{-1}b^2(s_n)=O_p(n^{-2})$.
The linear and quadratic error terms, controlled as in the proof of Lemma \ref{lemma: concentration}, therefore satisfy
$$
\begin{aligned}
&V_n(\tau_n)-V_n(s_n)-\mathbb E\{V_n(\tau_n)-V_n(s_n)\mid\bmX\}\\
&\quad=O_p\!\left(\frac{\sigma_{\varepsilon,n}}n
       +\sigma_{\varepsilon,n}^2 n^{-3/4}
       +\frac{\sigma_{\varepsilon,n}^2\tau_n}{\sqrt n}\right)
       =o_p(\sigma_{\varepsilon,n}^2\tau_n).
\end{aligned}
$$

Combining this with the preceding lower bound, we have $V_{n}(\tau_{n}) - V_{n}(1/\sqrt{n}) \geq \sigma_{\varepsilon, n}^{2}\tau_n\frac{\operatorname{Tr}(T)}{2} + o_{p}(\sigma_{\varepsilon, n}^{2}\tau_n)$ and hence $V_{n}(\tau_{n}) - V_{n}(1/\sqrt{n})  > 0$ with probability approaching one. Thus, according to the convexity of $V_{n}(t)$, we have $\widehat t_n^{\mathrm{REML}} \leq \tau_n$ with probability approaching one for any sequence such that $\sqrt{n}\tau_n \to \infty$ and $\tau_n\to0$, which implies $\widehat t_n^{\mathrm{REML}} = O_p(n^{-1/2})$. This completes the proof.

\end{proof}

\begin{proof}[Proof of Corollary \ref{cor:out_sample_optimality}]
Define
$$
Q_f(\bmx)
=
\{f_\star(\bmx)- f(\bmx)\}^2+\sigma_{\varepsilon,n}^2 .
$$
Then the random-design prediction risk can be written as
$$
\mathcal E_n^*(t)
=
\mathbb E\big[\mathbb{P}_XQ_{\widehat f^{\bmH}_t}\big]
=
\mathbb E\left[
\{y_{\rm new}-\widehat f^{\bmH}_t(\bmx_{\rm new})\}^2
\right],
$$
whereas the fixed-design in-sample prediction risk considered in Theorem \ref{thm: test error}
can be written as
$$
\mathcal E_n(t)
=
\mathbb E\big[\mathbb P_nQ_{\widehat f^{\bmH}_t}\mid\bmX\big]
=
\frac1n\sum_{i=1}^n
\mathbb E\left[
\{y_i^{\rm new}-\widehat f^{\bmH}_t(\bmx_i)\}^2
\mid \bmX
\right].
$$
Recall that $\widehat f^{\bmH}_t(\bmx) = f_0(\bmx) + h (\bmx,\bmX) \bm{H}^{\dagger} \left\{ \bm{I} - \exp\left( - t \bm{H} \right)\right\} \left\{\bm{y} - f_0(\bmX) \right\}$. Let $\cH$ be the RKHS associated with $h(\cdot,\cdot)$. Note that $h (\bmx,\bmX) \bm{H}^{\dagger} \left\{ \bm{I} - \exp\left( - t \bm{H} \right)\right\} \left\{\bm{y} - f_0(\bmX) \right\} \in \cH$ and
\begin{equation*}
    \begin{aligned}
        &\|h (\bmx,\bmX) \bm{H}^{\dagger} \left\{ \bm{I} - \exp\left( - t \bm{H} \right)\right\} \left\{\bm{y} - f_0(\bmX) \right\}\|^2_{\cH}\\
        &= \left\{\bm{y} - f_0(\bmX) \right\}^{\top}\left\{ \bm{I} - \exp\left( - t \bm{H} \right)\right\}\bm{H}^{\dagger} \left\{ \bm{I} - \exp\left( - t \bm{H} \right)\right\} \left\{\bm{y} - f_0(\bmX) \right\}
    \end{aligned}
\end{equation*}
Moreover,
\begin{equation}\label{eq: RKHS norm bound}
    \begin{aligned}
        &\mathbb{E}\left[\left\{\bm{y} - f_0(\bmX) \right\}^{\top}\left\{ \bm{I} - \exp\left( - t \bm{H} \right)\right\}\bm{H}^{\dagger} \left\{ \bm{I} - \exp\left( - t \bm{H} \right)\right\} \left\{\bm{y} - f_0(\bmX) \right\}\mid\bmX\right]\\
        & = \sum_{k=1}^{n}\lambda_k^{-1}\{1 - \exp(-t\lambda_k)\}^{2}[\bm{v}_k^{\top}\{f_{\star}(\bmX) - f_{0}(\bmX)\}]^2 + \sigma_{\varepsilon,n}^{2}\sum_{k=1}^{n}\lambda_k^{-1}\{1 - \exp(-t\lambda_k)\}^{2}\\
        &\leq \sum_{k=1}^{n}\frac{b_k^2}{\lambda_k} + Ct^{2}\sum_{k=1}^{n}\lambda_k\\
        &=O(1) + Cnt^{2}n^{-1}\sum_{k=1}^{n}h(\bmx_i,\bmx_i)\\
        &= O(1) + Cnt^{2}\{\mathbb{E}\{h(\bmx,\bmx)\} + o_p(1)\} = O(1)    \end{aligned}
\end{equation}
according to Condition \ref{ass: NTK spectral}.
Here $r=f_\star-f_0\in\cH$, $\sum_k b_k^2/\lambda_k\le\|r\|_{\cH}^2$, and the last bound is uniform for $0\le t\le M/\sqrt n$ for each fixed $M<\infty$, by boundedness of the kernel.
For any $C > 0$,
let
$$
\mathcal Q_{C}=\{Q_f: f = f_0 + h\ \text{with}\ \|h\|_{\cH}\leq C\}.
$$
Under the conditions $\mathbb{E}[f_{0}(\bmx)^4] < \infty$ and $\sup_{\bmx}|h(\bmx,\bmx)| < \infty$, the function class $\mathcal Q_{C}$ is $P$-Donsker according to Example 1.8.5, Theorem 2.6.14, and Lemma 2.6.20 of \cite{van1996weak}. 
Therefore,
\begin{equation}\label{eq: empirical process}
    \sup_{q\in \mathcal Q_{C}}
\left|
(\mathbb P_n-\mathbb{P}_X)q
\right|
=
O_{p}(n^{-1/2}).
\end{equation}
In addition, for any $\epsilon > 0$,
there is some $C$ such that $Q_{\widehat f^{\bmH}_{\widehat t_n^{\rm REML}}}\in \mathcal Q_{C}$ with probability larger than $1 - \epsilon$ according to \eqref{eq: RKHS norm bound} and Lemma \ref{lem: order of stop time}. Indeed, $\|\widehat f_t^{\bmH}-f_0\|_{\cH}^2$ is nondecreasing in $t$, so Markov's inequality and \eqref{eq: RKHS norm bound} control its supremum on $[0,M/\sqrt n]$. Then, \eqref{eq: empirical process} implies $(\mathbb P_n-\mathbb{P}_X)Q_{\widehat f^{\bmH}_{\widehat t_n^{\rm REML}}} = O_p(1/\sqrt{n})$.

To pass from the realized losses to the averaged risks defined above, let
$R=\|r\|_{\cH}$, $\kappa^2=\sup_{\bmx}h(\bmx,\bmx)$,
$\mathsf C_n=n^{-1}S_n^*S_n$, $\mathsf C=\mathbb E\{h(\cdot,\bmx)\otimes h(\cdot,\bmx)\}$, and
$\Delta_n=\|\mathsf C_n-\mathsf C\|_{\mathrm{op}}$, where $S_n$ is the evaluation operator used in Lemma~\ref{lem: emp NTK spectral} and $(u\otimes u)v=\langle u,v\rangle_{\cH}u$.
For $f=f_0+g$, the covariance identity gives
$$
|(\mathbb P_n-\mathbb P_X)Q_f|
\le\Delta_n(R+\|g\|_{\cH})^2,
\qquad
\mathbb E\Delta_n
\le\{\mathbb E\|\mathsf C_n-\mathsf C\|_{\mathrm{HS}}^2\}^{1/2}
\le\kappa^2/\sqrt n.
$$
Writing $\mathbb E_X$ for expectation over the training design, \eqref{eq: RKHS norm bound} and the bounded kernel therefore imply
$$
\sup_{0\le t\le M/\sqrt n}
|\mathcal E_n^*(t)-\mathbb E_X\mathcal E_n(t)|
\le C_M\mathbb E\Delta_n=O(n^{-1/2}).
$$
It remains to control the design fluctuation of $\mathcal E_n(t)$.
For positive semidefinite $A,B$, putting $P=e^{-uA}$ and $Q=e^{-uB}$ gives
$$
Ae^{-2uA}-Be^{-2uB}=PA(P-Q)+P(A-B)Q+(P-Q)BQ.
$$
For $u>0$, $\|P-Q\|_{\mathrm{op}}\le u\|A-B\|_{\mathrm{op}}$ by Duhamel's identity and
$\|PA\|_{\mathrm{op}},\|BQ\|_{\mathrm{op}}\le1/(eu)$.
Hence $\|Ae^{-2uA}-Be^{-2uB}\|_{\mathrm{op}}\le(1+2/e)\|A-B\|_{\mathrm{op}}$, also at $u=0$.
Since $b^2(2t)=\langle r,\mathsf C_ne^{-2nt\mathsf C_n}r\rangle_{\cH}$ and
$0\le\mu_2(2t)\le t\bar\lambda\le t\kappa^2$, the risk decomposition in the proof of Theorem~\ref{thm: test error} yields
$$
\sup_{0\le t\le M/\sqrt n}
|\mathcal E_n(t)-\mathbb E_X\mathcal E_n(t)|
\le (1+2/e)R^2(\Delta_n+\mathbb E\Delta_n)
+2\sigma_{\varepsilon,n}^2\kappa^2M/\sqrt n
=O_p(n^{-1/2}).
$$
Combining the last two bounds gives the same uniform order for
$|\mathcal E_n^*(t)-\mathcal E_n(t)|$.

Notice that
$\sqrt n\,\sigma_{\varepsilon,n}^2\to\infty$.
Then, by Lemma~\ref{lem: order of stop time},
\begin{equation}\label{eq: gap fix-random}
\left|
\mathcal E_n(\widehat t_n^{\rm REML})-\mathcal E_n^*(\widehat t_n^{\rm REML})
\right|
= O_p(1/\sqrt{n}) =
o_{\mathbb P}(\sigma_{\varepsilon,n}^2).
\end{equation}
By Theorem \ref{thm: test error},
$$
\frac{
\mathcal E_n(\widehat t_n^{\rm REML})
}{
\inf_{t\ge 0}\mathcal E_n(t)
}
\to 1
$$
in probability. Moreover, the comparison-time argument in the proof of Theorem~\ref{thm: test error} gives
$$
\inf_{t\ge 0}\mathcal E_n(t)
=
\sigma_{\varepsilon,n}^2\{1+o_{\mathbb P}(1)\}.
$$
Thus, we have
$$
\mathcal E_n(\widehat t_n^{\rm REML})
=
\sigma_{\varepsilon,n}^2\{1+o_{\mathbb P}(1)\}.
$$
Combining this with \eqref{eq: gap fix-random}, we obtain
$$
\mathcal E_n^*(\widehat t_n^{\rm REML})
=
\mathcal E_n(\widehat t_n^{\rm REML})
+
o_{\mathbb P}(\sigma_{\varepsilon,n}^2)
=
\sigma_{\varepsilon,n}^2\{1+o_{\mathbb P}(1)\}.
$$
Note that $\sigma_{\varepsilon,n}^2\leq\inf_{t\geq0}\mathcal{E}_n^*(t)\leq \mathcal E_n^*(\widehat t_n^{\rm REML})=
\sigma_{\varepsilon,n}^2\{1+o_{\mathbb P}(1)\}$. We have
$$
\frac{
\mathcal E_n^*(\widehat t_n^{\rm REML})
}{
\inf_{t\ge 0}\mathcal E_n^*(t)
}
=
\frac{
\sigma_{\varepsilon,n}^2\{1+o_{\mathbb P}(1)\}
}{
\sigma_{\varepsilon,n}^2\{1+o_{\mathbb P}(1)\}
}
\to 1
$$
in probability. This proves Corollary \ref{cor:out_sample_optimality}.
\end{proof}

\section{Additional Simulations}
\label{sec: supp simu}

We conduct numerical experiments to evaluate the early-stopping time estimated by REML with a Regression Transformer model \citep{born2023regression}. We generate $n_{\text{train}}=1{,}000$ training samples and $n_\text{test}=100$ test samples. For each sample $i$, the predictor is generated as a token sequence $\bm{X}_i=(\bm{x}_{i1},\ldots,\bm{x}_{iQ})^\top\in\mathbb{R}^{Q\times d}$, where the sequence length is $Q=5$ and the token dimension is $d=5$. Each token $\bm{x}_{iq}$ represents the covariates observed at position $q$ and is written as $\bm{x}_{iq}=(\bm{z}_{iq}^\top,\tau_q)^\top$, where $\bm{z}_{iq}=(z_{iq,1},\ldots,z_{iq,4})^\top\in\mathbb{R}^4$ contains four dynamic features and $\tau_q$ is a deterministic position feature equally spaced on $[-1,1]$. The dynamic features follow the autoregressive process $\bm{z}_{i1}\sim N(\bm{0},\bm{I}_4)$ and $\bm{z}_{iq}=0.7\bm{z}_{i,q-1}+\sqrt{1-0.7^2}\bm{\epsilon}_{iq}$ for $q=2,\ldots,Q$, where $\bm{\epsilon}_{iq}\sim N(\bm{0},\bm{I}_4)$. The response is generated as $y_i=f_\star(\bm{X}_i)+\epsilon_i$, with $\epsilon_i\sim N(0,0.25)$. The ground-truth function is constructed as $f_\star(\bm{X}_i)=0.45f_1(\bm{X}_i)+0.30f_2(\bm{X}_i)+0.70f_3(\bm{X}_i)+0.35f_4(\bm{X}_i)$, where $f_1(\bm{X}_i)=\sum_{q=1}^{Q}\{(0.18+0.08\tau_q)\sin(z_{iq,1})+0.12\cos(z_{iq,2})+0.10\tau_q\tanh(z_{iq,3})\}$, $f_2(\bm{X}_i)=(Q-1)^{-1}\sum_{q=1}^{Q-1}\tanh(\bm{z}_{iq}^{\top}\bm{z}_{i,q+1}/\sqrt{4})$, $f_3(\bm{X}_i)=\sum_{q=1}^{Q}a_{iq}\{0.60z_{iq,1}-0.40z_{iq,2}+0.30\sin(z_{iq,4})\}$ with $a_{iq}=\exp(\bm{z}_{iQ}^{\top}\bm{z}_{iq}/\sqrt{4}+0.75\tau_q)/\sum_{\ell=1}^{Q}\exp(\bm{z}_{iQ}^{\top}\bm{z}_{i\ell}/\sqrt{4}+0.75\tau_\ell)$, and $f_4(\bm{X}_i)=\sin(z_{i1,1}z_{iQ,2})+0.50\tanh(z_{i3,3}z_{iQ,4})$. This data-generating process makes the signal depend on nonlinear token-wise effects.

We train a Regression Transformer model with depth $L=2$, hidden width $w=1{,}000$, softmax attention, layer normalization, feed-forward layers with ReLU activation functions, mean pooling over tokens, and a linear readout \citep{vaswani2017attention, born2023regression}. The model is optimized using full-batch gradient descent with learning rate $\eta=10^{-2}$. We compare the finite-width Regression Transformer training dynamics with the corresponding infinite-width Transformer NTK gradient-flow solution, and evaluate performance using mean squared error (MSE) on both training and test datasets. We compute the infinite-width Transformer NTK using the algorithm in \citet{yang2020tensor}. To quantify performance, we consider two metrics: (1) test error; (2) computational time. The results are shown in Figure~\ref{fig: simu REML time transformer}.

Figure~\ref{fig: simu REML time transformer} shows that the  REML-guided early stopping time occurs near the onset of the test-error plateau, where further training yields diminishing improvements. In addition, the gradient descent trajectories closely track the corresponding NTK gradient flow trajectories for both training and test errors, especially before $10^5$ epochs. Specifically, REML selects $\widehat{t}_n^{\mathrm{REML}}=1{,}907$ epochs as the early stopping time, whereas the oracle stopping time is $25.6$ times longer, at $48{,}800$ epochs. The edf of this Regression Transformer model stopped at $\widehat{t}_n^{\mathrm{REML}}$ is 12.64. We also compare the REML stopping time with the validation-based stopping time by splitting the $1000$ training sample into training and validation subsets with a $2{:}1$ ratio \citep{prechelt1998early}. Validation-based early stopping requires evaluating validation loss over a checkpoint grid, whereas REML estimates the stopping time directly from the fixed training operator constructed at initialization. Compared with validation-based early stopping, REML achieves comparable test error, with a test error of $0.2071$, compared with $0.2092$ for validation. REML also substantially reduces computational time, requiring only $7.11$ minutes, compared with $306.50$ minutes for validation.  All computations were performed on a Dell Pro Max 16 workstation equipped with an Intel Core Ultra 9 285H CPU, Intel Arc integrated GPU, and 32 GB RAM. These results demonstrate that REML provides a principled criterion for early stopping. It reliably identifies a training time that achieves competitive predictive performance while substantially reducing computational time cost.

\clearpage

\section*{Supplementary Tables and Figures}

\begin{table}[ht!]
\centering
\caption{Comparison between classical \textit{static} random-effects representations for estimators with an explicit quadratic penalty $\rho\mathcal P(f)$, such as penalized splines and kernel ridge regression, and the proposed \textit{dynamic }random-effects representation for early-stopped gradient flow. In the static case, the fitted values can be written as BLUPs under $\bmy=\bmu_\rho+\bm\varepsilon_\rho$, where $\bmu_\rho\sim N(\bm 0,\sigma_{\varepsilon,\rho}^2\bm C_\rho)$, $\bm\varepsilon_\rho\sim N(\bm 0,\sigma_{\varepsilon,\rho}^2\bm I)$, and $\bm C_\rho=\rho^{-1}\bm C$ for a fixed covariance matrix $\bm C$ determined by the explicit quadratic penalty. For early-stopped gradient flow, the covariance is time-dependent, $\bm C_t=\exp(t\bm H)-\bm I$, where $\bm H$ is the fixed training operator.}
\label{tab:compare}
\resizebox{\textwidth}{!}{
\renewcommand{\arraystretch}{1.25}
\begin{tabular}{lcc}
\toprule
\textbf{Aspect} & \textbf{Explicit Quadratic-Penalty Estimators}  & \textbf{Early-Stopped Gradient Flow} \\
\midrule
Objective function & $\min_f \sum_{i=1}^n (y_i - f(\bmx_i))^2 + \rho\mathcal P(f)$  & $\min_f \sum_{i=1}^n (y_i - f(\bmx_i))^2$ \\
Regularization mechanism & Explicit penalty  & Implicit via early stopping \\
Covariance structure & Static $\bm C_\rho=\rho^{-1}\bm C$
& Dynamic $\displaystyle \bm C_t=\exp(t\bm H)-\bm I$ \\
Parameter &  Static penalty parameter $\rho$ & Dynamic training-time parameter $t$ \\
BLUP & Final regularized estimator  & Entire training trajectory \\
Spectral shrinkage & Static ridge-type shrinkage  & Dynamic exponential  shrinkage\\
\bottomrule
\end{tabular}}
\end{table}

\begin{table}[thb]
\centering
\caption{Empirical Type I Error Rates and Power of the Projected NTK-REML Score Test.}
\label{tab: type1_power}
\begin{tabular}{lccccc}
\toprule
 & $n=100$ & $n=200$ & $n=300$ & $n=400$ & $n=500$ \\
\midrule
Type I Error & 0.042 & 0.047 & 0.051 & 0.050 & 0.052 \\
Power        & 0.620 & 0.901 & 0.967 & 0.982 & 0.994 \\
\bottomrule
\end{tabular}
\end{table}

\newpage

\begin{figure}[ht]
    \centering
    \includegraphics[width=1\linewidth]{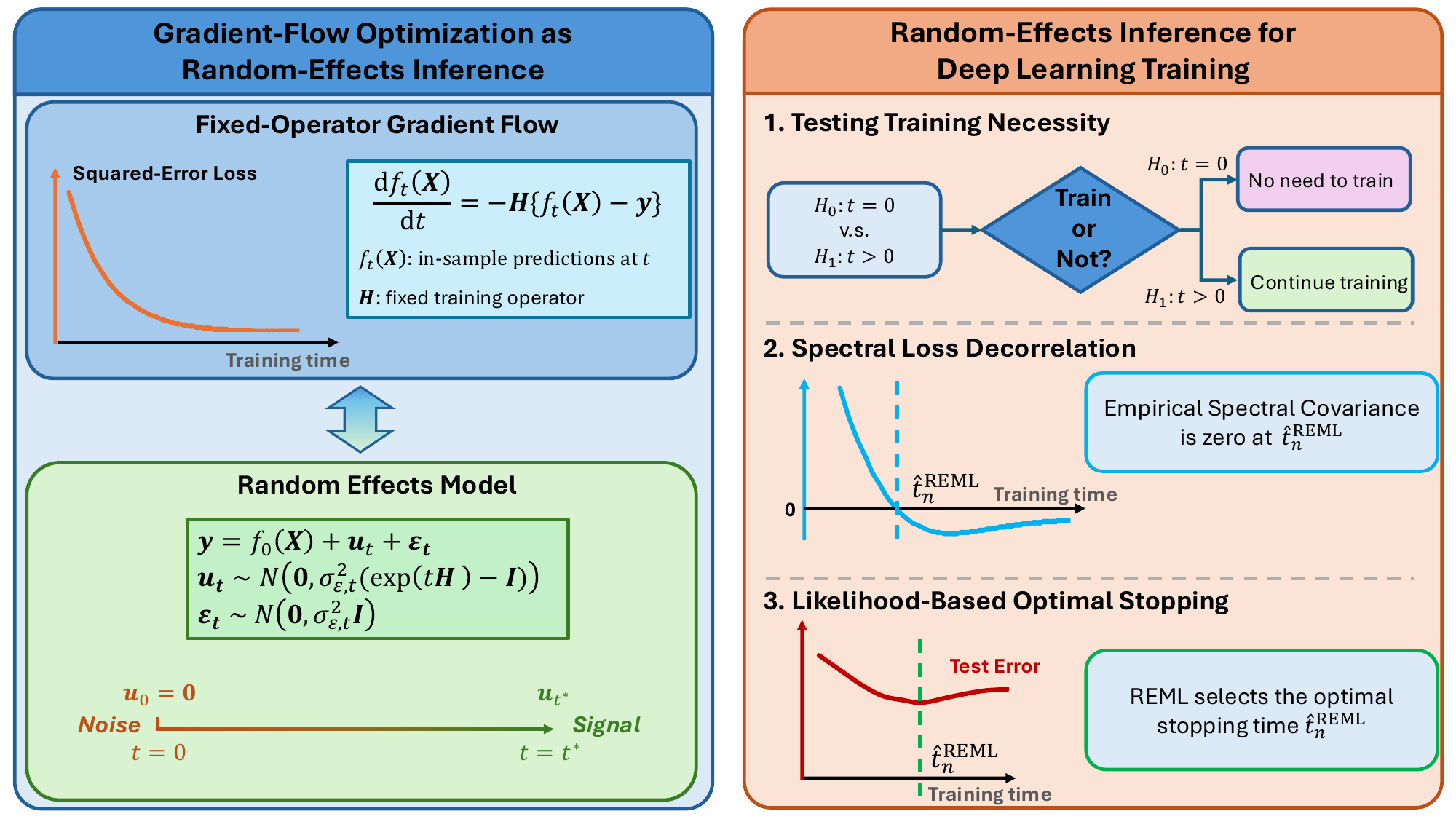}
    \caption{
Random-effects interpretation and inference-based training of deep learning models.
Left: Equivalence between fixed-operator gradient flow and a random-effects model, where the network output corresponds to the best linear unbiased predictor (BLUP) and training time acts as a variance component governing the allocation of variance from noise to signal.
Right: Inference-based training procedure. A statistical test determines whether training beyond initialization is necessary. Training then proceeds until spectral loss decorrelation is achieved, characterized by zero empirical spectral covariance (ESC). The corresponding stopping time is selected via restricted maximum likelihood (REML), yielding an asymptotically optimal prediction rule under the regularity conditions stated in Theorem~\ref{thm: test error} and Corollary~\ref{cor:out_sample_optimality}.
}
    \label{fig: nn diagram}
\end{figure}

\begin{figure}[!htbp]
    \centering
    \includegraphics[width=1.0\linewidth]{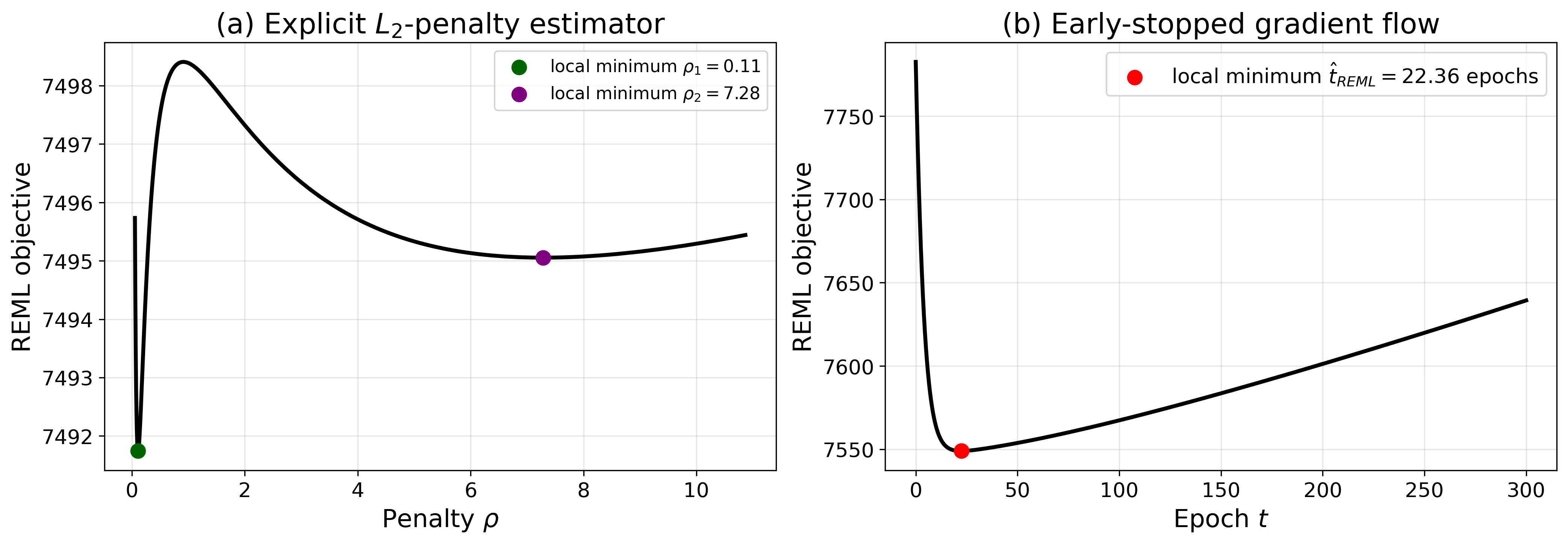}
    \caption{
    (a) The REML objective for explicit $L_2$-penalty estimators need not be convex and can have multiple local minima. The kernel is chosen to be a degree-3 polynomial kernel $k(\bmx,\bmx')=(2+0.01 \bmx^\top \bmx')^3$. The horizontal axis shows the penalty, and the vertical axis shows the corresponding REML objective. The green and purple dashed vertical lines mark two local minima. (b) Empirical illustration of convexity of the gradient-flow REML objective $Q(t)$ in the simulation setting. The horizontal axis shows training time on the epoch scale, and the vertical axis shows the REML objective. The red dashed vertical line marks the REML-guided stopping time $\widehat t_n^{\mathrm{REML}}$. 
    }
    \label{fig:supp_Q_convexity}
\end{figure}

\begin{figure}[!htbp]
    \centering
    \includegraphics[width=1\linewidth]{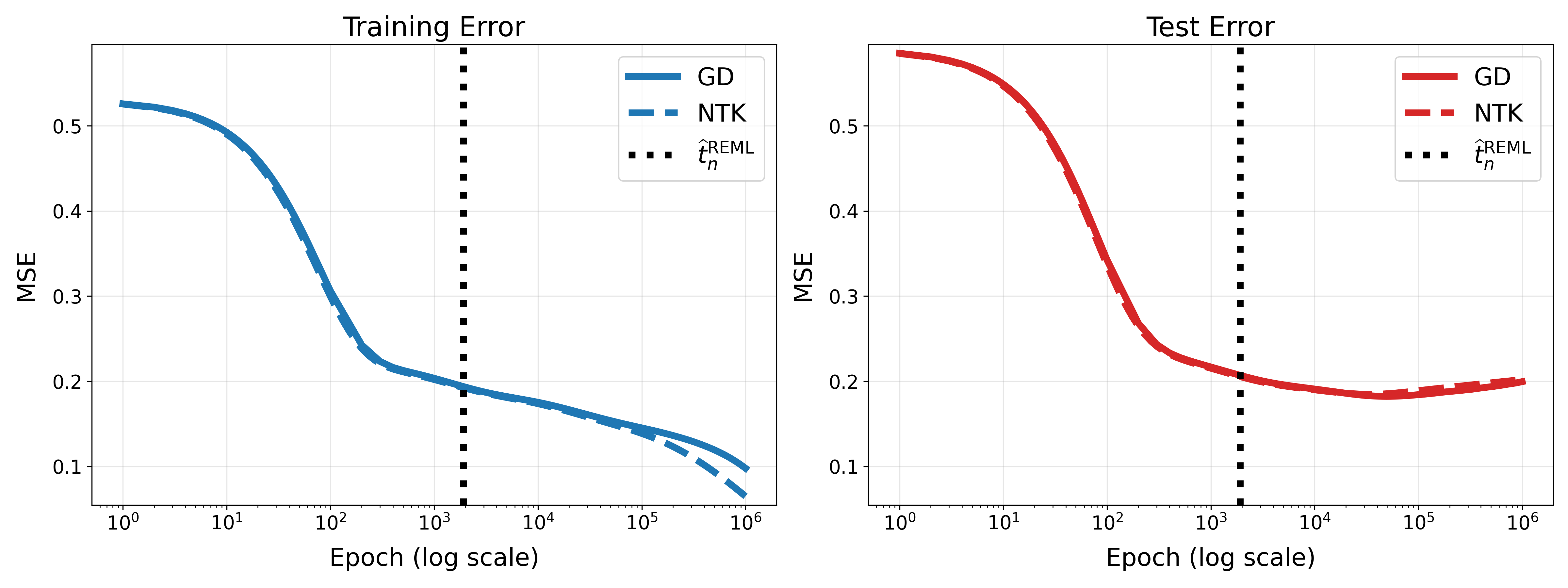}
    \caption{Comparison of gradient descent and NTK gradient flow with a Regression Transformer model. The vertical dotted line $\widehat{t}_n^{\rm REML}$ represents the  REML early-stopping time.}
    \label{fig: simu REML time transformer}
\end{figure}

\begin{figure}[!thb]
    \centering
    \includegraphics[width=1\linewidth]{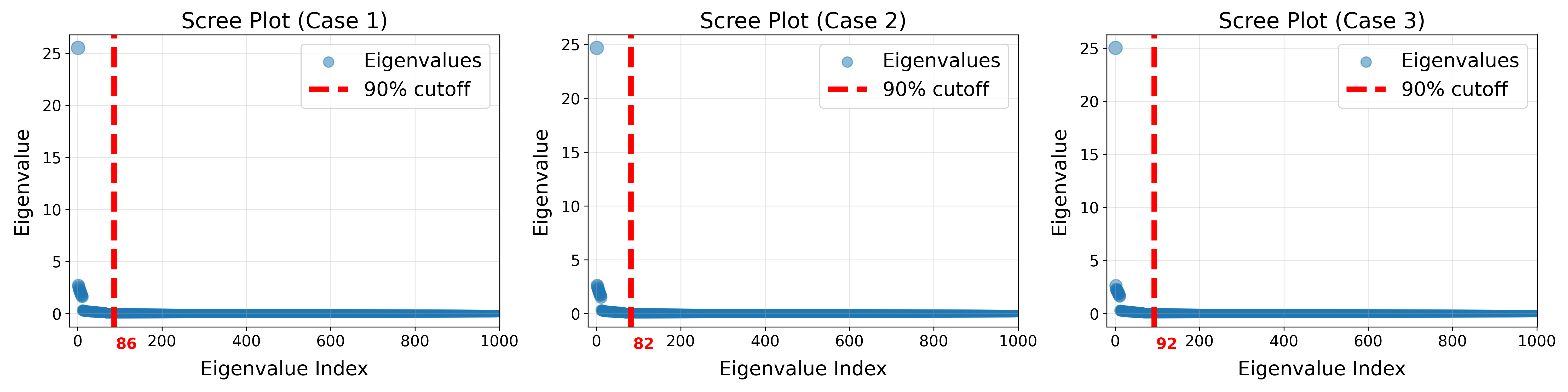}
    \caption{Scree plot of NTK matrices in Section \ref{sec: simu early stop}. Red dashed lines represent the eigenvalue index at which the cumulative sum of eigenvalues reaches 90\% of the total sum of eigenvalues.}
    \label{fig: scree_plot}
\end{figure}

\end{document}